\documentclass{article}

\usepackage[numbers]{natbib}
\usepackage[final]{neurips_2023}

\usepackage[utf8]{inputenc} % allow utf-8 input
\usepackage[T1]{fontenc}    % use 8-bit T1 fonts
\usepackage{hyperref}       % hyperlinks
\usepackage{url}            % simple URL typesetting
\usepackage{booktabs}       % professional-quality tables
\usepackage{amsfonts}       % blackboard math symbols
\usepackage{nicefrac}       % compact symbols for 1/2, etc.
\usepackage{microtype}      % microtypography
\usepackage{xcolor}         % colors

\usepackage{amsmath,amssymb,amsbsy,amsfonts,amsthm}
\usepackage{cancel}
\usepackage{tikz}
\usetikzlibrary{matrix,arrows,arrows.meta,calc,patterns,decorations.pathmorphing,decorations.pathreplacing,backgrounds,positioning,fit,shapes}
\usepackage{caption}
\usepackage{subcaption}
\usepackage{multirow}
\newcommand{\bydef}[1]{{#1}_{\textrm{def}}}

\newcommand{\MF}{\mathcal{F}}
\newcommand{\MX}{\mathcal{X}}
\newcommand{\MZ}{\mathcal{Z}}

\theoremstyle{plain}

\newtheorem{definition}{Definition}
\newtheorem{proposition}{Proposition}

\title{BQ-NCO: Bisimulation Quotienting for Efficient Neural Combinatorial Optimization}

\author{%
Darko Drakulic, Sofia Michel, Florian Mai\thanks{IDIAP Research Institute, Work done during an internship at Naver Labs Europe}, Arnaud Sors, Jean-Marc Andreoli \\
  Naver Labs Europe\\
  \texttt{firstname.lastname@naverlabs.com} \\
}

\begin{document}

\maketitle

\begin{abstract}
Despite the success of neural-based combinatorial optimization methods for end-to-
end heuristic learning, out-of-distribution generalization remains a challenge. In
this paper, we present a novel formulation of Combinatorial Optimization Problems (COPs) as Markov Decision Processes (MDPs) that effectively leverages common symmetries of COPs to improve out-of-distribution robustness. Starting
from a direct MDP formulation of a constructive method, we introduce a
generic way to reduce the state space, based on Bisimulation Quotienting (BQ) in MDPs. Then, for COPs with a recursive nature, we specialize the bisimulation and show how the reduced state exploits the symmetries of these problems and facilitates MDP solving. Our approach is principled and we prove that an optimal policy for the proposed BQ-MDP actually solves the associated COPs. We illustrate our approach on five classical problems: the Euclidean and Asymmetric Traveling Salesman, Capacitated Vehicle Routing, Orienteering and Knapsack Problems. Furthermore, for each problem, we introduce a simple attention-based policy network for the BQ-MDPs, which we train by imitation of (near) optimal solutions of small instances from a single distribution. We obtain new state-of-the-art results for the five COPs on both synthetic and realistic benchmarks. Notably, in contrast to most existing neural approaches, our learned policies show excellent generalization performance to much larger instances than seen during training, without any additional search procedure. Our code is available at: \href{https://github.com/ddrakulic/BQNCO/}{url}. 
\end{abstract}

\section{Introduction} %
Combinatorial Optimization Problems (COPs) are crucial in many application domains such as transportation, energy, logistics, etc. Because they are generally NP-hard \citep{cook_combinatorial_1997}, their resolution at real-life scales is mainly done by problem-specific heuristics, which heavily rely on expert knowledge. %\citep{boussaid_survey_2013}. 
Neural combinatorial optimization (NCO) is a relatively recent line of research that focuses on using deep neural networks to learn such heuristics from data, possibly exploiting regularities in problem instances of interest \citep{bengio_machine_2021, cappart_combinatorial_2021}. 
Among NCO methods, the so-called constructive approaches view the process of building a solution incrementally as a sequential decision making problem, which can naturally be modeled with Markov Decision Processes (MDPs). 
Although many previous works have successfully used this strategy, the process of formulating an appropriate MDP (esp. the state and action spaces) is usually specific to each problem (e.g. \citep{kool_attention_2019} for routing problems, \citep{zhang_learning_2020} for jobshop scheduling and \citep{jiang_solving_2021} for bin packing). These design choices have a considerable impact when solving the MDP. In particular, exploiting the COP's symmetries can boost the efficiency and generalization of neural solvers (e.g. \citep{kwon_pomo_2020, kim_sym-nco_2022} leverage the symmetries of routing problems in Euclidean graphs). 

In this paper, we present a generic and principled framework to cast any COP as an MDP and then a way to account for common COPs' symmetries to design a more efficient MDP. More precisely, given a user-defined {\it solution space}, we show how to automatically derive a {\it direct MDP}, where the state is a partial solution and the actions are the construction steps. While the application of this framework to some COPs encompasses some previously proposed MDPs (e.g \citep{kool_attention_2019}, for routing problems), to the best of our knowledge, we are the first to provide a way to derive an MDP for any COP and prove the equivalence between the optimal MDP policies and solving the COP. 

Next, noting the limitations of using the partial solutions as states, we introduce a generic {\it bisimulation} mapping that allows to reduce the state space. In particular, we show that for problems that satisfy a {\it recursion property}, the bisimulation simply maps a partial solution to a new (induced) instance, which corresponds to the remaining subproblem when the partial solution is fixed. As many partial solutions can induce the same instance, the resulting bisimulation quotiented (BQ-)MDP has a significantly smaller state space. Also, it enables more efficient learning by avoiding to independently learn the policy at states which are equivalent for the COP at hand. In contrast to previous works (\citep{kwon_pomo_2020, kim_sym-nco_2022}), the symmetries that we exploit are not linked to the Euclidean nature of the graphs but to the recursive property of the problems, which is very common in CO as it includes the Optimality Principle of Dynamic Programming \citep{bellman_theory_1954, bertsekas_dynamic_vol1_2012}.

We illustrate our framework on five well-known COPs: the Euclidean and Asymmetric Traveling Salesman Problems (TSP, ATSP), the Capacitated Vehicle Routing Problem (CVRP), the Orienteering Problem (OP) and the Knapsack Problem (KP). Furthermore, we propose a simple transformer-based architecture \citep{vaswani_attention_2017} for these problems, well-suited to the BQ-MDPs, and requiring only minor variations to work for the five problems.
% Although close to previous transformer-based architectures for CO, \todo{our model has a key differentiating inductive bias: unlike previous models, the amount of computation of the model is proportionate to the subproblem size and therefore complexity}.
For each problem, we train our policy by imitation of expert trajectories derived from (near) optimal solutions of small instances sampled from a single distribution. We test on both synthetic and realistic benchmarks of varying size and node distributions. Our model provides new state-of-the-art results on the five problems and exhibits excellent out-of-distribution generalization, especially to larger instances, a well-recognized challenge for NCO methods~\citep{joshi_learning_2022,manchanda_generalization_2022}.
Notably with a single greedy rollout, our learned policies outperform state-of-the-art end-to-end learning-based approaches for instances with more than 200 nodes. We show that we can further improve the performance with a beam-search (using more computation) while we can significantly speed-up the execution (at the cost of slight performance drop) by replacing the quadratic transformer model by a linear attention-based model, the PerceiverIO \cite{jaegle_perceiver_2022}.

In summary, our contributions are as follows: 1) We present a generic and principled framework to derive a direct MDP given any COP with minimal requirements; 2) We propose a method to reduce the direct MDPs via symmetry-focused bisimulation quotienting and define an explicit bisimulation for the class of recursive COPs; 3) We design an adequate transformer-based architecture for the BQ-MDPs, with only minor adaptations to work for the TSP (both Euclidean and asymmetric versions), CVRP, OP and KP; 4) We achieve state-of-the-art generalization performance on these five problems, significantly out-performing other neural-based constructive methods.

\section{Combinatorial Optimization as a Markov Decision Problem} % 
\label{sec:co-as-mdp}
In this section, we propose a generic framework to represent any Combinatorial Optimization Problem (COP) as a Markov Decision Process (MDP) amenable to standard Machine Learning techniques. 

We denote a COP instance by: 
$$\min_{x\in X}f(x),$$ 
where $X$ is the finite, non-empty set of {\em feasible solutions}, while the {\em objective} $f$ is a real-valued function whose domain contains $X$. The complexity of CO is due to the cardinality of $X$, which, although finite, is generally exponential in the problem size.
Constructive approaches to CO build a solution sequentially by growing a partial solution at each step. A required assumption of constructive heuristics (although often left implicit) is that the feasibility of the final solution can be ensured through tractable conditions on the partial solutions at each step of the construction process.

%%%%%%%%%%%%%%%%%%%%%%%%%%%%%
\subsection{Solution Space}
\label{sec:solution-space}
%%%%%%%%%%%%%%%%%%%%%%%%%%%%
Let us denote by $\MX$ the set of all possible {\em partial solutions} for a given COP, so that it contains the feasible solutions of {\it any} instance. We assume that $\MX$ is equipped with an operation $\circ$ having a neutral element $\epsilon$. Informally, $\epsilon$ is the ``empty'' partial solution and if $x,y$ are partial solutions, then $x{\circ}y$ denotes the result of applying the sequence of construction steps yielding $x$ followed by that yielding $y$. We denote by $\MZ{\subset}\MX$ the subset of partial solutions obtained from $\epsilon$ by just one construction step. By identification, we call the elements of $\MZ$ {\it steps} and assume that any partial (a fortiori feasible) solution can be obtained as the composition of a sequence of steps. We can now define the structure of the solution space as required by our framework.
\begin{definition}[Solution Space]
\label{def:solution-space}
% A {\em solution space} is a tuple $(\MX,\circ,\epsilon,\MZ)$ where $(\MX,\circ,\epsilon)$ forms a monoid (see Sec.~\ref{anx:background-monoids} for background), and $\MZ{\subset}\MX{\setminus}\{\epsilon\}$, whose elements are called {\em steps}, is a generator of $\MX$. We make the stronger assumption that any element of $\MX$ has a {\em finite} positive number of step decompositions:
A {\em solution space} is a tuple $(\MX,\circ,\epsilon,\MZ)$ where $(\MX,\circ,\epsilon)$ forms a monoid (see Appendix~\ref{anx:background-monoids} for background), and the step set $\MZ{\subset}\MX{\setminus}\{\epsilon\}$ is a generator of $\MX$, such that any element of $\MX$ has a finite positive number of step decompositions:
\begin{equation}
\label{eqn:state-decomposition}
\forall x\in\MX,\;0<|\{z_{1:n}\in\MZ^{n}:\;x=z_1\circ\cdots\circ z_n\}|<\infty.
\end{equation}
\end{definition}
Thus, solving a COP instance $(f,X)$ given a solution space $\MX$ containing $X$ amounts to finding a sequence of steps $z_{1:n}$ such that $z_1{\circ}\cdots{\circ}z_n{\in}\arg\min_{x\in X}f(x)$, a task which can naturally be modeled as an MDP, as shown below. 
%\sm{Maybe move the following to later since we don't/can't explain why it's essential at this point} It is essential that the same MDP be applicable to all the instances of a COP, so the solution space should be defined problem-wise (for a goal-directed MDP), or even apply to multiple problems (for a multi-task MDP). This usually imposes that $\MZ$ and $\MX$ are infinite, but the finiteness condition in \eqref{eqn:state-decomposition} avoids the pitfalls which may result.

%%%%%%%%%%%%%%%%%%%%%%%%%%%%%%%%%%%%%%%%%%%%%%%%%%%%%%%%%%%%%%%%%%%%%%%%%%%
%\paragraph{A prototypical example} 
\paragraph{Examples of Solution Spaces.}
%%%%%%%%%%%%%%%%%%%%%%%%%%%%%%%%%%%%%%%%%%%%%%%%%%%%%%%%%%%%%%%%%%%%%%%%%%%
We provide examples of solution spaces for two classical COPs.
In the Euclidean Traveling Salesman Problem (TSP), an instance is defined by a set of nodes (including an origin) associated with points in a Euclidean space and the goal is to find the shortest tour that starts and ends at the origin and visits each other node exactly once. A partial solution for the TSP can be represented by a sequence of nodes. Therefore we can define $\MX_{\text{TSP}}$ as the set of finite sequences of nodes.
In the Knapsack Problem (KP), an instance is defined by a set of items associated with weight and value features, and the goal is to select a subset of these items such that the sum of their weights does not exceed a given capacity while their cumulated value is maximized. We can define $\MX_{\text{KP}}$ as the set of finite subsets of items.
For TSP (resp. KP), the {\it operator} $\circ$ is sequence concatenation (resp. set union), the {\it neutral element} $\epsilon$ is the empty sequence (resp. set) and a {\it step} is a sequence (resp. set) of length 1.
%Note that each $\MX$ is infinite and defined problem-wise so that the feasible solutions {\it any} instance of the problem are included in the corresponding $\MX$. 
Note that these problems admit other (e.g. graph-based rather than set or sequence-based) solution spaces. In fact, a solution space is not intrinsic to a problem, nor vice-versa.

%%%%%%%%%%%%%%%%%%%%%%%%%%%%%%%%%%%%%%%%%%%%%%%%%%%%%%%%%%%%%%%%%%%%%%%%%%%
\subsection{The Direct MDP}  % of an instance in a solution space}
%%%%%%%%%%%%%%%%%%%%%%%%%%%%%%%%%%%%%%%%%%%%%%%%%%%%%%%%%%%%%%%%%%%%%%%%%%%
Given a solution space $(\MX,\circ,\epsilon,\MZ)$, we consider the set $\MF_\MX$ of instances $(f,X)$ where the set of feasible solutions $X$ is a finite, non-empty subset of $\MX$, and the objective function $f{\in}\mathbb{R}^\MX$ is defined on $\MX$ rather than just $X$, {\it i.e. it is well-defined for any partial solution}.
Now, given an instance $(f,X){\in}\MF_\MX$, we can derive its {\em direct} MDP, denoted $\mathcal{M}_{(f,X)}$, as follows. Its {\bf state space} is the set $\bar{X}{=}\{x{\in}\MX{:}\,\exists y{\in}\MX,\; x{\circ} y{\in} X\}$ of partial solutions which can potentially be expanded into a feasible one. Its {\bf action space} is $\MZ{\cup}\{\epsilon\}$, i.e. an action is either a step or the neutral action. Its {\bf transitions} are deterministic, and can be represented as a labeled transition system with the following two rules, where the label of each transition consists of its action-reward pair, placed respectively above and below the transition arrow:

\begin{equation}
\label{eqn:direct-transitions}
x\xrightarrow[f(x)-f(x\circ z)]{z}x{\circ}z \;\textrm{ if }\; x{\circ} z\in \bar{X}
\hspace{1cm} \text{and} \hspace{1cm}
x\xrightarrow[\,0\,]{\epsilon}x \;\textrm{ if }\; x\in X.
\end{equation}
Here, $x{\in}\bar{X}$ is a state and $z{\in}\MZ$ is a step. The condition in each rule determines whether the action is allowed. Thanks to the structure of the solution space captured by Def.~\ref{def:solution-space}, the direct MDP of an instance $(f,X)$ has three key properties, proved in Appendix~\ref{anx:proof-properties}. (1) From any state, the number of allowed actions is finite (even if $\MZ$ is infinite), therefore $\mathcal{M}_{(f,X)}$ belongs to the simpler class of discrete-action MDPs. (2) From any state, there is always at least one allowed action (i.e. there are no dead-end states). This assumes that one can guarantee if a partial solution can be expanded into a feasible one, as required of the valid states in $\bar{X}$. This is a common assumption of constructive heuristics (often left implicit) and it avoids the complications associated with dead ends in MDPs~\citep{kolobov_theory_2012}. (3) In any infinite trajectory in $\mathcal{M}_{(f,X)}$, the number of transitions involving a {\it step} action is finite, while all the other transitions involve the {\it neutral} action. Since the neutral action yields a null reward, this means that the return of a trajectory is well defined, without having to discount the rewards. Also, if $a_{1:\infty}$ is the (infinite) sequence of actions of the trajectory, since all but finitely many of them are neutral actions, their composition $a_1{\circ}a_2{\circ}{\cdots}$ is well defined and is called the {\em outcome} of the trajectory. 
%It is also its stationary state. An alternative formulation of $\mathcal{M}_{(f,X)}$ as an episodic MDP with no infinite trajectory is also possible, but makes further assumptions and is less elegant.
We can now establish the main result of this section, proved in Appendix~\ref{anx:proof-soundness}: 
\begin{proposition}[Soundness of the Direct MDP]
\label{prop:soundness-direct}
Given a solution space $\MX$ and an instance $(f,X){\in}\MF_\MX$, let $\mathcal{M}_{(f,X)}$ be its direct MDP. The set $\arg\min_{x\in X}f(x)$ is exactly the set of $x$ such that there exists an optimal policy $\pi$ for $\mathcal{M}_{(f,X)}$ where $x$ is the outcome of a trajectory starting at $\epsilon$ under policy $\pi$.
\end{proposition}
This result states the exact correspondence between the optimal solutions of a COP instance and the optimal policies of its direct MDP. Thus, the vast corpus of techniques developed to search for optimal policies in MDPs in general is applicable to solving COPs. The direct MDP encompasses many previously proposed MDPs, where the state is a partial solution and each action consists of adding an element to the partial solution, e.g. \citep{khalil_learning_2017} for graph problems and \citep{bresson_transformer_2021} for the TSP.

% \section{Bisimulation Quotienting for Combinatorial Optimization} %
\section{Bisimulation Quotienting}

\subsection{State information and symmetries}

Let ($\MX,\circ,\epsilon,\MZ)$ be a solution space and $\mathcal{M}_{(f,X)}$ the direct MDP of an instance $(f,X){\in}\MF_\MX$. In a trajectory of $\mathcal{M}_{(f,X)}$, observe that each non neutral action $z$ ``grows'' the state from $x$ to $x{\circ}z$, as defined by \eqref{eqn:direct-transitions}. For many COPs, this is counter-intuitive, because it implies that the state information increases while the number of allowed actions generally decreases. For example, in the TSP, the state contains the sequence of visited nodes, which grows at each step, while the allowed actions are the set of unvisited nodes, which shrinks at each step. At the end, the state may carry the most information (a full-grown feasible solution), but it is not used in any decision, since there is a single allowed action anyway (the neutral action). To address this mismatch between the information carried by the state and the complexity of the decision it supports, we seek to define a new state which captures only the information needed for the continuation of the construction process. To do so, we observe that a partial solution $y{\in}\MX$ can be seen both as an {\it operator on partial solutions} $x{\mapsto}x{\circ}y$ which {\it grows} its operand, or, alternatively, as the following {\it operator on instances}, which {\it reduces} its operand:
\[
(f,X){\mapsto}(f*y, X*y) 
\hspace{0.3cm}\text{with:}\hspace{0.3cm} (f*y)(x)=f(y\circ x) 
 \hspace{0.3cm}\text{and}\hspace{0.3cm}  X*y = \{x|y\circ x\in X\}.
\]
In fact, $(f{*}y,X{*}y)$ is the so-called the {\it tail subproblem}~\citep{bertsekas_dynamic_vol1_2012} of instance $(f,X)$ after partial solution $y$ has been constructed. We observe that for a given tail subproblem $(f',X'){\in}\MF_\MX$, there can be many combinations of an instance $(f,X){\in}\MF_\MX$ and partial solution $x{\in}\bar{X}$ such that $(f{*}x,X{*}x){=}(f',X')$.
For example in the TSP, the tail subproblem consists of finding the shortest path from the last node $e$ of the partial solution back to the origin $o$ of the instance, which visits the set $I$ of not-yet-visited nodes. Therefore we can see that any TSP instance with same origin $o$ and any node sequence ending at $e$ with not-yet-visited nodes $I$ will lead to the same tail subproblem. This is a strong {\it symmetry} with respect to the MDP policy: intuitively it means that an optimal policy should produce the same action for these {\em infinitely} many (instance--partial solution) pairs. Treating them as distinct states, as in the direct MDP, forces a training procedure to learn the underlying symmetry in order to map them into a common representation.

\subsection{Bisimulation Quotiented MDP}
\label{sec:bq-mdp}
In order to leverage the above ``reduction operator'' view and its potential symmetry, we define the following ``reduced'' MDP $\mathcal{M}$. Its {\bf state space} is the set $\MF_\MX$ of instances. Its {\bf action space} is $\MZ{\cup}\{\epsilon\}$, as in the direct MDP. Its {\bf transitions} are also deterministic, and can be expressed with two rules, dual of those in~\eqref{eqn:direct-transitions}:
\begin{equation}
\label{eqn:adjoint-transitions}
(f,X)\xrightarrow[f(\epsilon)-f(z)]{z}(f*z,X*z) \;\textrm{ if }\; X*z\not=\emptyset
\hspace{0.4cm} \text{and} \hspace{0.4cm}
(f,X)\xrightarrow[\;0\;]{\epsilon}(f,X) \;\textrm{ if }\; \epsilon\in X.
\end{equation}
This MDP is defined at the level of the whole solution space rather than an individual instance. Now, the mapping $\boldsymbol{\Phi}_{(f,X)}{:}\bar{X}{\mapsto}\MF_\MX$ from the direct to the reduced states defined by $\boldsymbol{\Phi}_{(f,X)}(x){=}(f{*}x,X{*}x)$ for all $x{\in}\bar{X}$ is a {\it bisimulation} between $\mathcal{M}_{(f,X)}$ and $\mathcal{M}$ (background in Appendix~\ref{anx:background-bisimulation} and a proof in Appendix
~\ref{anx:proof-bisimulation}). Formally, $\mathcal{M}$ is isomorphic to the quotient of the direct MDP by the bisimulation (precise statement and proof in Appendix~\ref{anx:proof-bisimulation}), hence is called the Bisimulation Quotiented (BQ-)MDP. By the bisimulation property, for any direct state $x$, the action-reward sequences spawned from $x$ in $\mathcal{M}_{(f,X)}$ and from $\boldsymbol{\Phi}_{(f,X)}(x)$ in $\mathcal{M}$ are identical. Therefore, there is a one-to-one correspondence between the trajectories of $\mathcal{M}_{(f,X)}$ starting at $\epsilon$ and those of $\mathcal{M}$ starting at $\boldsymbol{\Phi}_{(f,X)}(\epsilon){=}(f,X)$. Hence, the analogue of Prop.~\ref{prop:soundness-direct} holds for the BQ-MDP:
\begin{proposition}[Soundness of the BQ-MDP]
\label{prop:soundness-adjoint}
Let $\mathcal{M}$ be the BQ-MDP of a solution space $\MX$, and $(f,X){\in}\MF_\MX$ an instance. The set $\arg\min_{x\in X}f(x)$ is exactly the set of $x$ such that there exists an optimal policy $\pi$ for $\mathcal{M}$ where $x$ is the outcome of a trajectory starting at $(f,X)$ under policy $\pi$.
\end{proposition}

%%%%

\paragraph{Impact on model architecture.} Although both the direct and BQ-MDPs are equivalent in terms of solving their associated COP, their practical interpretation leads to major differences. In the direct MDP view, it would not make sense to learn each instance-specific MDP separately. Instead a generic MDP conditioned on an input instance is learnt, similar to goal-conditioned Reinforcement Learning~\citep{schaul_universal_2015,liu_goal-conditioned_2022} (the goal is here the input instance). A typical policy model architecture consists of an encoder in charge of computing an embedding of the input instance and a decoder that takes the instance embedding and the current partial solution to compute the next action (Fig.~\ref{fig:architecture} left), e.g. the Attention Model \citep{kool_attention_2019} or PointerNetworks \citep{vinyals_pointer_2015}.
%The encoder and decoder modules typically involve some form of self-attention, while the decoder also uses cross-attention to incorporate the instance embedding output by the encoder. 
In the rollout of a trajectory, the encoder needs only be invoked once since the instance does not change throughout the rollout. For the BQ-MDP, only one, unconditional MDP is learnt for the whole solution space. The model can be simpler since the distinction between encoder and decoder vanishes (Fig.~\ref{fig:architecture} right). On the other hand, the whole model must be applied to a new input instance at each step of a rollout.

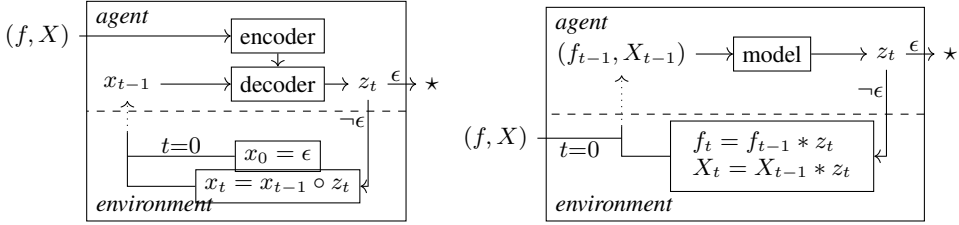
\begin{figure}[h!]
\begin{center}
\begin{tikzpicture}[scale=.8]
\begin{scope}[local bounding box=a]
\node at (0,0.4) {}; \node at (4.9,-2.9) {};
\node[draw] (encoder) at (3,0) {\small encoder};
\node[draw] (decoder) at (3,-.8) {\small decoder};
\node (z) at (4.5,-.8) {\small $z_t$};
\coordinate (z1) at (5.1,-.8); \draw (z) -- (z1) node[midway,above,inner sep=1] {\small $\epsilon$};
\node (x) at (.5,-.8) {\small $x_{t-1}$};
\coordinate (x0) at (.5,-1.6);
\node[draw] (x1) at (3,-2) {\small $x_0=\epsilon$};
\node[draw] (x2) at (3,-2.5) {\small $x_{t}=x_{t-1}\circ z_t$};
\draw[->] (encoder) -- (decoder);
\draw (z) -- (z|-x2) node[near start,left,inner sep=1]{\small $\neg\epsilon$}; \draw[->] (z|-x2) |- (x2); \draw (x2) -| (x0); \draw (x1) -- node[midway,above,inner sep=1] {\small $t{=}0$} (x1-|x0); \draw[dotted,->] (x0) -- (x);
\draw[->] (x) -- (decoder); \draw[->] (decoder) -- (z);
\draw ($(a.north east)+(.03,.03)$) rectangle ($(a.south west)+(-.03,-.03)$);
\draw[dashed] (a.east)--(a.west);
\node[anchor=north west] at ($(a.north west)$) {\small\em agent};
\node[anchor=south west] at ($(a.south west)$) {\small\em environment};
\end{scope}
\node[anchor=east] (inst) at (-.3,0) {\small $(f,X)$}; \draw[->] (inst) -- (encoder);
\node[anchor=west] (sol) at (5.3,-.8) {$\star$}; \draw[->] (z1) -- (sol);
\end{tikzpicture}
\begin{tikzpicture}[scale=.8]
\begin{scope}[local bounding box=a]
\node at (0,0.6) {}; \node at (4.9,-2.6) {};
\node[draw] (model) at (3,0) {\small model};
\node (z) at (4.9,0) {\small $z_t$};
\coordinate (z1) at (5.5,0); \draw (z) -- (z1) node[midway,above,inner sep=1] {\small $\epsilon$};
\node (fX) at (.5,0) {\small $(f_{t-1},X_{t-1})$};
\coordinate (via) at (.5,-1.2);
\node[draw] (fX0) at (3,-1.7) {\small $\begin{array}{l}f_t=f_{t-1}*z_t\\X_t=X_{t-1}*z_t\end{array}$};
\draw (z) -- (z|-fX0) node[near start,left,inner sep=1]{\small $\neg\epsilon$}; \draw[->] (z|-fX0) |- (fX0); \draw (fX0) -| (via); \draw[dotted,->] (via) -- (fX);
\draw[->] (fX) -- (model); \draw[->] (model) -- (z);
\draw ($(a.north east)+(.03,.03)$) rectangle ($(a.south west)+(-.03,-.03)$);
\draw[dashed] (a.east)--(a.west);
\node[anchor=north west] at ($(a.north west)$) {\small\em agent};
\node[anchor=south west] at ($(a.south west)$) {\small\em environment};
\end{scope}
\node[anchor=east] (inst) at (-.9,-1.4) {\small $(f,X)$}; \draw (inst) -- node[midway,below,inner sep=1] {\small $t{=}0$} (inst-|via);
\node[anchor=west] (sol) at (5.7,0) {$\star$}; \draw[->] (z1) -- (sol);
\end{tikzpicture}
\end{center}
\caption{\label{fig:architecture}Policy model architectures for the  direct MDP (left) and BQ-MDP (right).}
\end{figure}

\subsection{Instance parametrization and recursion}
\label{sec:recursion}

For a given COP, instances are generally described by a set of parameters, which are used as input to the policy network. In the BQ-MDP, the instance is updated at each step $z$, according to the equations $X'{=}X{*}z$ and $f'{=}f{*}z$. In order to implement the BQ-MDP, a key requirement is that $(f',X')$ can be represented in the {\it same} parametric space as $(f, X)$. In fact, this is the case for COPs that satisfy the {\it tail-recursion property}: after applying a number of construction steps to an instance, the remaining tail subproblem is itself an instance of the original COP. This is a very common property in CO and includes in particular the Optimality Principle of Dynamic Programming \citep{bellman_theory_1954, bertsekas_dynamic_vol1_2012}: all problems that are amenable to dynamic programming satisfy the tail-recursion property. For these tail-recursive COPs, the bisimulation simply maps a partial solution to the tail subproblem instance it induces.

\noindent \textbf{Application to the KP and path-TSP.} 
One can see that the KP naturally satisfies the tail-recursion property. Consider an instance of the KP with capacity $c$ and items $(w_i,v_i)_{i\in I}$ described by their weight $w_i$ and value $v_i$. A partial solution is a subset $J{\subset}I$ of items. In this case, the bisimulation $\boldsymbol{\Phi}$ maps $J$ to a new KP instance with capacity $c{-}\sum_{j\in J}w_j$ and item set $(w_i,v_i)_{i\in I\setminus J}$ (see Appendix~\ref{anx:kp}).
For the TSP, we introduce a slightly more general problem, path-TSP, where the goal is to find a shortest path between an origin and a destination node visiting a set of customer nodes, instead of a tour. The tail subproblem for a partial path $x_{1:k}$ consists of finding the shortest path from $x_k$ to the destination that visits all the remaining nodes, hence it is also a path-TSP instance. Thus path-TSP is tail-recursive, and TSP is simply a sub-problem of path-TSP where origin and destination are associated with the same Euclidean point. A similar reasoning holds for CVRP and OP, leading to path-CVRP and path-OP (see in Appendix~\ref{anx:cvrp} and ~\ref{anx:op}).

\section{Policy Learning}
\label{sec:neuralnet}
We now describe our proposed transformer-based policy network for the BQ-MDPs. For simplicity, we focus here on the path-TSP; the models for path-CVRP, path-OP and KP differ only slightly and are presented in Appendix~\ref{anx:cvrp},~\ref{anx:op} and~\ref{anx:kp}.

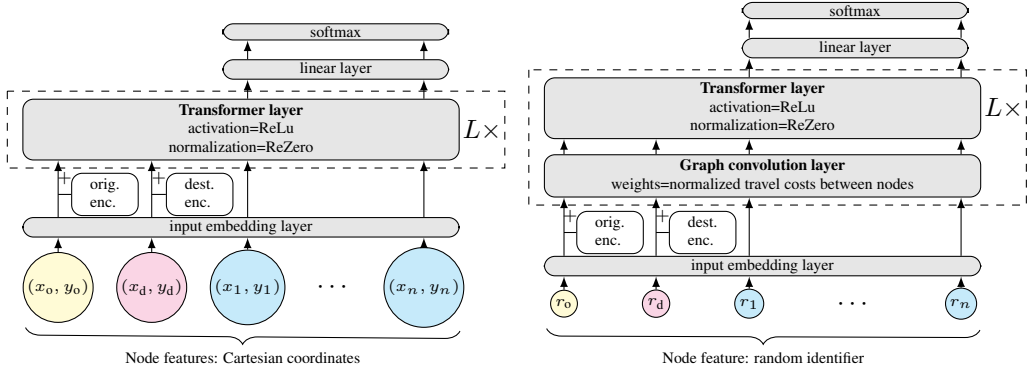
\begin{figure*}
\begin{center}
\begin{subfigure}{.49\textwidth}
\begin{tikzpicture}[
>=latex,
n/.style={draw,circle,inner sep=1},
t/.style={draw,rectangle,rounded corners,fill=black!10,inner sep=1}
]
\matrix[column sep=.3cm]{
\node[n,fill=yellow!20] (ori) {\tiny $(x_{\textrm{o}},y_{\textrm{o}})$}; &
\node[n,fill=magenta!20] (dst) {\tiny $(x_{\textrm{d}},y_{\textrm{d}})$}; &
\node[n,fill=cyan!20] (x1) {\tiny $(x_1,y_1)$}; &
\node {$\cdots$}; &
\node[n,fill=cyan!20] (xn) {\tiny $(x_n,y_n)$};\\
};
\draw[decoration={brace,amplitude=.2cm,mirror},decorate] ([shift={(-.1cm,-.2cm)}]ori.south west) -- node[label={[below,yshift=-.3cm]{\tiny Node features: Cartesian coordinates}}] {} ([shift={(.1cm,-.2cm)}]xn.south east);
\coordinate (hemb) at (0,.8); \coordinate (henc) at (0,2.1);
\coordinate (hdec) at (0,2.9); \coordinate (hsoft) at (0,3.4);
\coordinate (midall) at ($(ori)!.5!(xn)$);\coordinate (midx) at ($(x1)!.5!(xn)$);
\coordinate (mid1) at ($(ori)!.5!(dst)$); \coordinate (mid2) at ($(dst)!.5!(x1)$);
\node[t,minimum width=5.8cm] (emb) at (hemb-|midall) {\tiny input embedding layer};
\node[t,minimum width=5.8cm] (enc) at (henc-|midall) {\tiny \begin{tabular}{c}\textbf{Transformer layer}\\ activation=ReLu\\ normalization=ReZero\end{tabular}};
\draw[dashed] ([shift={(-.2cm,.1cm)}]enc.north west) rectangle ([shift={(.6cm,-.1cm)}]enc.south east) {};
\node at ([xshift=.3cm]enc.east) {$L\times$};
\node[t,minimum width=3cm] (dec) at (hdec-|midx) {\tiny linear layer};
\node[t,minimum width=3cm] (soft) at (hsoft-|midx) {\tiny softmax};
\coordinate (hembenc) at ($(emb.north)!.4!(enc.south)$);
\node[t,fill=white] (orisig) at (hembenc-|mid1) {\tiny \begin{tabular}{l}orig.\\enc.\end{tabular}};
\node[t,fill=white] (dstsig) at (hembenc-|mid2) {\tiny \begin{tabular}{l}dest.\\enc.\end{tabular}};
\draw[->] (ori) -- (emb.south-|ori);
\draw[->] (dst) -- (emb.south-|dst);
\draw[->] (x1) -- (emb.south-|x1);
\draw[->] (xn) -- (emb.south-|xn);
\draw[->] (emb.north-|ori) -- (enc.south-|ori);
\draw[->] (emb.north-|dst) -- (enc.south-|dst);
\draw[->] (emb.north-|x1) -- (enc.south-|x1);
\draw[->] (emb.north-|xn) -- (enc.south-|xn);
\draw[->] (enc.north-|x1) -- (dec.south-|x1);
\draw[->] (enc.north-|xn) -- (dec.south-|xn);
\draw[->] (dec.north-|x1) -- (soft.south-|x1);
\draw[->] (dec.north-|xn) -- (soft.south-|xn);
\draw (orisig) -- (ori|-orisig) node[midway,label={[above,inner sep=0]:\tiny $+$}] {};
\draw (dstsig) -- (dst|-dstsig) node[midway,label={[above,inner sep=0]:\tiny $+$}] {};
\end{tikzpicture}
\end{subfigure}
\begin{subfigure}{.49\textwidth}
\begin{tikzpicture}[
>=latex,
n/.style={draw,circle,inner sep=1},
t/.style={draw,rectangle,rounded corners,fill=black!10,inner sep=1}
]
\matrix[column sep=.85cm]{
\node[n,fill=yellow!20] (ori) {\tiny $r_{\textrm{o}}$}; &
\node[n,fill=magenta!20] (dst) {\tiny $r_{\textrm{d}}$}; &
\node[n,fill=cyan!20] (x1) {\tiny $r_1$}; &
\node {$\cdots$}; &
\node[n,fill=cyan!20] (xn) {\tiny $r_n$};\\
};
\draw[decoration={brace,amplitude=.2cm,mirror},decorate] ([shift={(-.1cm,-.2cm)}]ori.south west) -- node[label={[below,yshift=-.3cm]{\tiny Node feature: random identifier}}] {} ([shift={(.1cm,-.2cm)}]xn.south east);
\coordinate (hemb) at (0,.5); \coordinate (hgnn) at (0,1.7); \coordinate (henc) at (0,2.6);
\coordinate (hdec) at (0,3.4); \coordinate (hsoft) at (0,3.9);
\coordinate (midall) at ($(ori)!.5!(xn)$);\coordinate (midx) at ($(x1)!.5!(xn)$);
\coordinate (mid1) at ($(ori)!.5!(dst)$); \coordinate (mid2) at ($(dst)!.5!(x1)$);
\node[t,minimum width=5.8cm] (emb) at (hemb-|midall) {\tiny input embedding layer};
\node[t,minimum width=5.8cm] (gnn) at (hgnn-|midall) {\tiny \begin{tabular}{c}\textbf{Graph convolution layer}\\weights=normalized travel costs between nodes\end{tabular}};
\node[t,minimum width=5.8cm] (enc) at (henc-|midall) {\tiny \begin{tabular}{c}\textbf{Transformer layer}\\ activation=ReLu\\ normalization=ReZero\end{tabular}};
\draw[dashed] ([shift={(-.2cm,.1cm)}]enc.north west) rectangle ([shift={(.6cm,-.1cm)}]gnn.south east) {};
\node at ([xshift=.3cm]enc.east) {$L\times$};
\node[t,minimum width=3cm] (dec) at (hdec-|midx) {\tiny linear layer};
\node[t,minimum width=3cm] (soft) at (hsoft-|midx) {\tiny softmax};
\coordinate (hembgnn) at ($(emb.north)!.4!(gnn.south)$);
\node[t,fill=white] (orisig) at (hembgnn-|mid1) {\tiny \begin{tabular}{l}orig.\\enc.\end{tabular}};
\node[t,fill=white] (dstsig) at (hembgnn-|mid2) {\tiny \begin{tabular}{l}dest.\\enc.\end{tabular}};
\draw[->] (ori) -- (emb.south-|ori);
\draw[->] (dst) -- (emb.south-|dst);
\draw[->] (x1) -- (emb.south-|x1);
\draw[->] (xn) -- (emb.south-|xn);
\draw[->] (emb.north-|ori) -- (gnn.south-|ori);
\draw[->] (emb.north-|dst) -- (gnn.south-|dst);
\draw[->] (emb.north-|x1) -- (gnn.south-|x1);
\draw[->] (emb.north-|xn) -- (gnn.south-|xn);
\draw[->] (gnn.north-|ori) -- (enc.south-|ori);
\draw[->] (gnn.north-|dst) -- (enc.south-|dst);
\draw[->] (gnn.north-|x1) -- (enc.south-|x1);
\draw[->] (gnn.north-|xn) -- (enc.south-|xn);
\draw[->] (enc.north-|x1) -- (dec.south-|x1);
\draw[->] (enc.north-|xn) -- (dec.south-|xn);
\draw[->] (dec.north-|x1) -- (soft.south-|x1);
\draw[->] (dec.north-|xn) -- (soft.south-|xn);
\draw (orisig) -- (ori|-orisig) node[midway,label={[above,inner sep=0]:\tiny $+$}] {};
\draw (dstsig) -- (dst|-dstsig) node[midway,label={[above,inner sep=0]:\tiny $+$}] {};
\end{tikzpicture}
\end{subfigure}
\end{center}
\caption{\label{fig:bq-model}BQ-MDP policy model architecture for the TSP (left) and ATSP (right); origin in yellow, destination in red and the other nodes in blue.}
\end{figure*}
\noindent \textbf{Neural architecture.}
For the Euclidean path-TSP (Fig.~\ref{fig:bq-model} left), the Cartesian coordinates of each node (including origin and destination) are embedded via a linear layer. The remainder of the model is based on \citet{vaswani_attention_2017} with the following differences.
First, we do not use positional encoding since the input nodes have no order. Instead, a learnable origin (resp. destination) encoding vector is added to the feature embedding of the origin (resp. destination) node.
Second, we use ReZero~\citep{bachlechner_rezero_2021} normalization, which leads to more stable training and better performance in our experiments. 
Finally, a linear layer projects the output of the last attention block into a vector of size $N$, from which unfeasible actions, corresponding to the origin and destination nodes, are masked out, before applying a softmax operator so as to interpret the scalar node values for all allowed nodes as action probabilities. Note the absence of autoregressive decoder layers.
In the asymmetric TSP, node coordinates are not given, but the travelling cost between nodes is provided through a (possibly asymmetric) matrix.
The policy model for path-ATSP (Fig.~\ref{fig:bq-model} right) differs slightly from that for path-TSP. Initially, a randomly generated identifier is assigned to each node, which is embedded by a linear layer. The cost matrix is then incorporated by using a weighted graph convolution operation in each layer of the model before applying the attention mechanism. The edge weights of the graph are obtained by some standard normalization of the cost matrix.
The use of random node identifiers as model input is justified in ATSP by the absence of node specific features. In problems where such features are available, a random identifier can still be added as extra node feature, as that often improves performance.

\noindent \textbf{Trajectory generation.} 
We train our model by imitation of expert trajectories, using a cross-entropy loss. Such trajectories are extracted from pre-computed (near) optimal solutions for instances of a relatively small and fixed size. Even for hard problems, we leverage the fact that it is generally possible to efficiently solve {\t small} instances (e.g. using MILP solvers). Of course the quality of the learned policy will be assessed  on larger and therefore more challenging instances. 
%obtained by handcrafted heuristics on sets of problem instances which are either synthetically generated or obtained from benchmarks, as described in Section~\ref{sec:experiments}. We use instances of homogeneous (relatively small) size $N$.
Note that optimal solutions are not directly in the form of trajectories, i.e. sequence of construction steps. While Prop.~\ref{prop:soundness-adjoint} guarantees that a trajectory exists for any solution, it is usually not unique. 
%Besides, optimal solutions are costly, so we seek to make the most out of each of them.
In the TSP, an optimal tour corresponds to two possible trajectories (one being the reverse of the other). In the CVRP, each subtour similarly corresponds to two possible trajectories, and in addition, the different orders of the subtours lead to different trajectories. We noticed that this final order has an impact on the performance (see Appendix~\ref{anx:expert-solutions}). On the other hand, any sub-sequence from a fixed trajectory can be interpreted as the solution of a sub-instance. Note that these sub-instances will vary both in size and node distribution, therefore by training on them, we implicitly encourage the model to work well across sizes and node distributions, and generalize better than if such variations were not seen during the training.

%\paragraph{Complexity}
\noindent \textbf{Complexity.}
Because of the quadratic complexity of self-attention, and the fact that we call our policy at each construction step, the total complexity of our model, for an instance of size $N$, is $\mathcal{O}(N^3)$, whereas closely related transformer-based models such as the TransformerTSP \citep{bresson_transformer_2021} and the Attention Model \citep{kool_attention_2019} have a total complexity of $\mathcal{O}(N^2)$. We experimentally show in Sec.~\ref{sec:experiments} that our model still provides better-quality solutions faster than these baselines. This is because most neural models are called many times per instance in practice, typically for sampling from the learned policy or within a beam search. In contrast, we observe that doing a single rollout of our policy yields excellent results. Intuitively, given the complexity (NP-hardness) of the targeted COPs, it makes sense to spend as much computation as possible at each time step, provided the whole rollout is fast enough.
% at each time step leads to better results, although the final policy still needs to be reasonably fast.
% \todo{However, in practice these models are also called many times (e.g. for sampling or within a beam search, etc) whereas with a one-shot (greedy) application of our model we get a better performance, in the end faster.} 
On the other hand, the quadratic attention bottleneck can be addressed by using some linear attention models. In our experiments, we show that replacing the Transformer by the linear PerceiverIO \cite{jaegle_perceiver_2022} significantly accelerates inference while still providing a competitive performance.

\section{Related Work} %
\label{sec:related-work}
\noindent{\bf Generic frameworks for CO.}
Generic frameworks have been crucial in the development and adoption of CO since they allow to formulate a COP in a certain format and then give it as an input to a (black box) solver. Among the most widespread frameworks are Mixed-Integer Programming~\citep{conforti_integer_2014} and Constraint Programming~\citep{rossi_handbook_2006}. On the other hand, Markov Decision Processes are a powerful and ubiquitous tool to model sequential decision making problems and make them amenable to Machine Learning. While using MDPs to solve CO is not new, there are very few attempts to make it a general framework. \citet{khalil_learning_2017} makes a first attempt to define a generic MDP for greedy heuristics on graph problems. \citet{drori_learning_2020} also focuses on graph problems and provides a flexible and efficient GNN-based architecture to train with Reinforcement Learning (RL). Unlike these methods, our approach does not assume a graph or any structure on the COPs to define the (direct and BQ) MDP and we are the first to prove the equivalence between our proposed MDPs and solving the COP.

%\paragraph{NCO approaches}
\noindent{\bf Neural-based constructive heuristics.}
Many NCO approaches construct solutions sequentially, via auto-regressive models. Starting with the seminal work by \citet{vinyals_pointer_2015}, which proposed the Pointer network trained in a supervised way, \citet{bello_neural_2017}
trained the same model by RL 
for the TSP and \citet{nazari_reinforcement_2018} adapted it for the CVRP.
\citet{kool_attention_2019} introduced an Attention-based encoder-decoder Model (AM) trained with RL to solve several routing problems. This architecture was reused, with a few extensions, in POMO \citep{kwon_pomo_2020} and further combined with efficient search procedures in \citep{choo_simulation-guided_2022}. TransformerTSP \citep{bresson_transformer_2021} uses a similar architecture with a different decoder for the TSP. 
%\fm{Is it fair to state the following?}All these previously mentioned approaches construct solutions sequentially, which is the most common approach, making them compatible with our work in principle.
%However, there exist different approaches to constructing solutions:
While most of these auto-regressive constructive approaches used RL, another line of work focuses on learning, in a supervised way, a heatmap of solution segments:
\citet{nowak_divide_2018} trained a Graph Neural Network to output an adjacency matrix, which is converted into a feasible solution using beam search; \citet{joshi_efficient_2019} followed a similar framework and trained a Graph Convolutional Network instead, that was used by~\citep{fu_generalize_2021} as well. Recently \citet{sun_difusco_2023} proposed DIFUSCO, a graph-based diffusion model to output the heatmap. Although this non-autoregressive methods are promising for large instances since they avoid the sequential decoding, they generally rely on sophisticated search methods (such as Monte Carlo Tree Search) in order to get a high-quality solution from the heatmap. Our paper explores the less common combination of supervised (imitation) learning to train a policy (not a heatmap) and achieves excellent performance without any search. In the context of planning, a similar strategy was successfully used in~\citep{groshev_learning_2018} to learn generalized reactive policies on small planning problems instances while generalizing well to larger ones.

\noindent{\bf Step-wise approaches.}
Instead of encoding an instance once and decoding sequentially as in the AM~\citep{kool_attention_2019}, \citet{peng_deep_2020} proposed to update the encoding after each subtour completion for CVRP while \citet{xin_step-wise_2021} updated it after each node selection in TSP. \citet{xin_multi-decoder_2021} extended on this idea by introducing the Multi-Decoder Attention Model (MDAM) which contains a special layer to efficiently approximate the re-embedding process. Since MDAM constitutes the most advanced version, we employ it as a baseline in our experiments.

% In practice, their architecture is the most similar to ours for the TSP. However, thanks to our principled bisimulation quotienting, we obtain a superior representation for CVRP: in contrast to our approach, their CVRP architecture only provides censored information by omitting the remaining vehicle capacity and simply restricting the state to the nodes whose demand is below the remaining capacity. 

%\paragraph{Generalizable NCO}
\noindent{\bf Generalizable NCO.}
Generalization to different instance distributions, and esp. larger instances, is regarded as one of the major challenges for current NCO approaches~\citep{joshi_learning_2022, mazyavkina_reinforcement_2021}.
% generalizable with respect to instance size
\citet{fu_generalize_2021} trained a Graph Convolution model in a supervised manner on small graphs and used it to solve large TSP instances, by applying the model on sampled subgraphs and using an expensive MCTS search to improve the final solution (Att-GCN+MCTS). While this method achieves excellent generalization on TSP instances, MCTS requires a lot of computing resources and is essentially a post-learning search strategy.
\citet{geisler_generalization_2022} investigate the robustness of NCO solvers through adversarial attacks and find that existing neural solvers are highly non-robust to out-of-distribution examples and conclude that one way to address this issue is through adversarial training. In particular, \citet{xin_generative_2022} trains a GAN to generate instances that are difficult to solve for the current model. 
\citet{manchanda_generalization_2022, son_meta-sage_2023, zhou_towards_2023} take a different approach and leverage meta-learning to learn a model that is easily adaptable to new distributions. In DIMES, \citet{qiu_dimes_2022} combines meta-learning with a continuous parametrization of the candidate solutions that allows to solve large-scale TSP instances. 
Accounting for symmetries in a given COP is a powerful idea to boost the generalization performance of neural solvers. For example, for problems on Euclidean graphs, Sym-NCO~\citep{kim_sym-nco_2022} make use of solution symmetries, to enforce robust representations while \citep{kwon_pomo_2020} exploits it as part of their loss function during training and at inference time to augment the set of solutions. Note that the above adversarial or meta-learning strategies and data augmentation are orthogonal to our approach of providing an efficient MDP which has beneficial effects irrespective of the training method. Finally, some hybrid methods have been successful in solving large-scale routing problems \cite{li_learning_2021, zong_rbg_2022, hou_generalize_2022}; these ate based on the divide-and-conquer principle: they learn how to split large instances into smaller ones, that are then efficiently solved by external specialized solvers.

\section{Experimental Evaluation} % 
\label{sec:experiments}
To evaluate the effectiveness of our method, we test it on four routing problems - Euclidean TSP, CVRP, OP and Asymmetric TSP and one non-routing problem - KP (details and results for the latter are in Appendix~\ref{anx:kp}). For all routing problems, we train our model and all baselines on synthetic instances of size 100, following the same procedure of data generation as in \cite{kool_attention_2019}. For asymmetric TSP, we use Euclidean distance matrices and randomly break symmetries. We choose graphs of size 100 because it is the largest size for which (near) optimal solutions are still reasonably fast to obtain, and such training datasets are commonly used in the literature. Then, we test the trained models on synthetic instances of size 100, 200, 500 and 1K generated from the same distribution, as well as the standard \href{http://comopt.ifi.uni-heidelberg.de/software/TSPLIB95/}{TSPLib} and \href{http://vrp.galgos.inf.puc-rio.br/index.php/en/}{CVRPLib} datasets. Except for perceiver-based models, we limit the subgraphs to the 250 nearest neighbours of the origin node at test time. This reduces inference time while not significantly impacting the solution quality (see Appendix~\ref{anx:knns}).

%%%%%%%%%%%%%%%%%%%%%%%%%%%%%%%%%%%%%%%%%%%%%%%%%%%%%%%%%%%%%%%%%%%%%%%%%%%
%\paragraph{Hyperparameters and training procedure}
%\vspace{-0.4cm}
\noindent\textbf{Hyperparameters and training.}
%%%%%%%%%%%%%%%%%%%%%%%%%%%%%%%%%%%%%%%%%%%%%%%%%%%%%%%%%%%%%%%%%%%%%%%%%%%
We use the same hyperparameters for all problems. The model has 9 layers, each built with 12 attention heads with embedding size of 192 and dimension of feed-forward layer of 512. Our model is trained by imitation of expert trajectories, using a cross-entropy loss. Solutions of these problems are obtained by using the Concorde solver \citep{applegate_concorde_2015} for TSP, LKH heuristic \cite{helsgaun_extension_2017} for ATSP and CVRP, and EA4OP heuristic~\cite{kobeaga_efficient_2018} for OP. We use a dataset of 1 million solutions. To sample trajectories out of this dataset, we note that in the case of TSP, any sub-path of the optimal tour is also an optimal solution to the associated path-TSP sub-problem, hence amenable to our path-TSP model. We therefore form minibatches by first sampling a number $n$ between $4$ and $N$ (path-TSP problems with less than $4$ nodes are trivial), then sampling sub-paths of length $n$ from the initial solution set. At each epoch, we sample a sub-path from each solution. By sampling subsequences among all possible infixes of the optimal solutions, we automatically get an augmented dataset, which some previous models had to explicitly design their model for~\citep{kwon_pomo_2020}. We use a similar sampling strategy for CVRP, OP and KP (see Appendix~.\ref{anx:cvrp}, \ref{anx:op} and~\ref{anx:kp}). 
Batches of size 1024 are formed, and the model is trained for 500 epochs. We use Adam~\citep{kingma_adam_2015} as optimizer with an initial learning rate of $7.5\mathrm{e}{-4}$ and decay of 0.98 every 50 epochs. 
 
%%%%%%%%%%%%%%%%%%%%%%%%%%%%%%%%%%%%%%%%%%%%%%%%%%%%%%%%%%%%%%%%%%%%%%%%%%%
%\paragraph{Evaluation}
%\vspace{-0.4cm}
\noindent\textbf{Baselines and test datasets.}
%%%%%%%%%%%%%%%%%%%%%%%%%%%%%%%%%%%%%%%%%%%%%%%%%%%%%%%%%%%%%%%%%%%%%%%%%%%
We compare our model with the following methods: OR-Tools \citep{perron_or-tools_2022}, LKH \citep{helsgaun_extension_2017}, and Hybrid Genetic Search (HGS) for the CVRP~\citep{vidal_hybrid_2022} as SOTA non-neural methods; DIFUSCO+2opt, Att-GCN+MCTS and SGBS as hybrid methods; and AM, TransformerTSP, MDAM, POMO, DIMES and Sym-NCO as deep learning-based constructive methods. Description of all baseline models are provided in Sec.~\ref{sec:related-work}. 
For all deep learning baselines we use the corresponding model \textit{trained on graphs of size 100} and the best decoding strategy\footnote{Att-GCN+MCTS is originally trained on graph sizes of 50. For DIMES and DIFUSCO authors provide generalization tables just for two types of decoding (greedy+2opt for DIFUSCO and RL+sampling for DIMES) and on some test datsets, so we report them in our results}. 
Following the same procedure as in \cite{fu_generalize_2021}, we generate four test datasets with graphs of sizes 100, 200, 500 and 1000. 
For CVRP, we use capacities of 50, 80, 100 and 250, respectively. In addition, we report the results on TSPLib instances with up to 4461 nodes and all CVRPLib instances with node coordinates in the Euclidean space. 
For all models, we report the optimality gap and the inference time. The optimality gap for TSP is based on the optimal solutions obtained with Concorde. For CVRP, although HGS gives better results than LKH, we use the LKH solutions as a reference to compute the "optimality" gap, in order to be consistent (and easily comparable) with previous works. While the optimality gap is easy to compute and compare, running times are harder to compare since they strongly depend on the implementation platforms (Python, C++), hardware (GPU, CPU), parallelization, batch size, etc. 
%Therefore, we also report the number of solutions generated by each of the constructive deep learning models.
In our experiments, we run all deep learning models on a single Nvidia Tesla V100-S GPU with 24GB memory, and other solvers on Intel(R) Xeon(R) CPU E5-2670 with 256GB memory, in one thread.

%%%%%%%%%%%%%%%%%%%%%%%%%%%%%%%%%%%%%%%%%%%%%%%%%%%%%%%%%%%%%%%%%%%%%%%%%%%
%\paragraph{Results}
\noindent\textbf{Results.}
%%%%%%%%%%%%%%%%%%%%%%%%%%%%%%%%%%%%%%%%%%%%%%%%%%%%%%%%%%%%%%%%%%%%%%%%%%%
Table \ref{tab:results} summarizes our results for Euclidean TSP, CVRP and OP as well as asymmetric TSP. For all problems, our model shows superior generalization on larger graphs, even with a greedy decoding strategy, which generates a single solution while all other baselines generate several hundreds (and select the best among them). In terms of running time, with greedy decoding, our model is slightly slower that Sym-NCO baseline, competitive with POMO, and significantly faster than other models. Beam search further improves the optimality gaps, but as expected, it takes more time. On the other hand, our model can be accelerated by replacing the quadratic attention blocks by the PerceiverIO architecture \citep{jaegle_perceiver_2022} (more details in Appendix~\ref{anx:perceiver}). This results in a considerable reduction of inference time for larger graphs, at the cost of some performance drop. Even with this trade-off, our model achieves remarkable performance compared with other NCO baselines. Figure~\ref{fig:optgap_time} shows optimality gap versus running time for our models and other baseline models. Our models clearly outperform other models in terms of generalization to larger instances. The only models that are competitive with ours are Att-GCN+MCTS and DIFUSCO+2opt, but both are hybrid methods with an expensive search on top of the output of the neural network, and they are 2 to 15 times slower while being designed for TSP only. 
Finally, in addition to the synthetic datasets, we test our models on TSPLib and VRPLib instances, which are of varying graph sizes, node distributions, demand distributions and vehicle capacities. Table~\ref{tab:results} shows that our model significantly outperforms the end-to-end baseline policies even with the greedy decoding strategy.

\begin{figure*}[h!]
    \centering
    \includegraphics[width=0.8\textwidth]{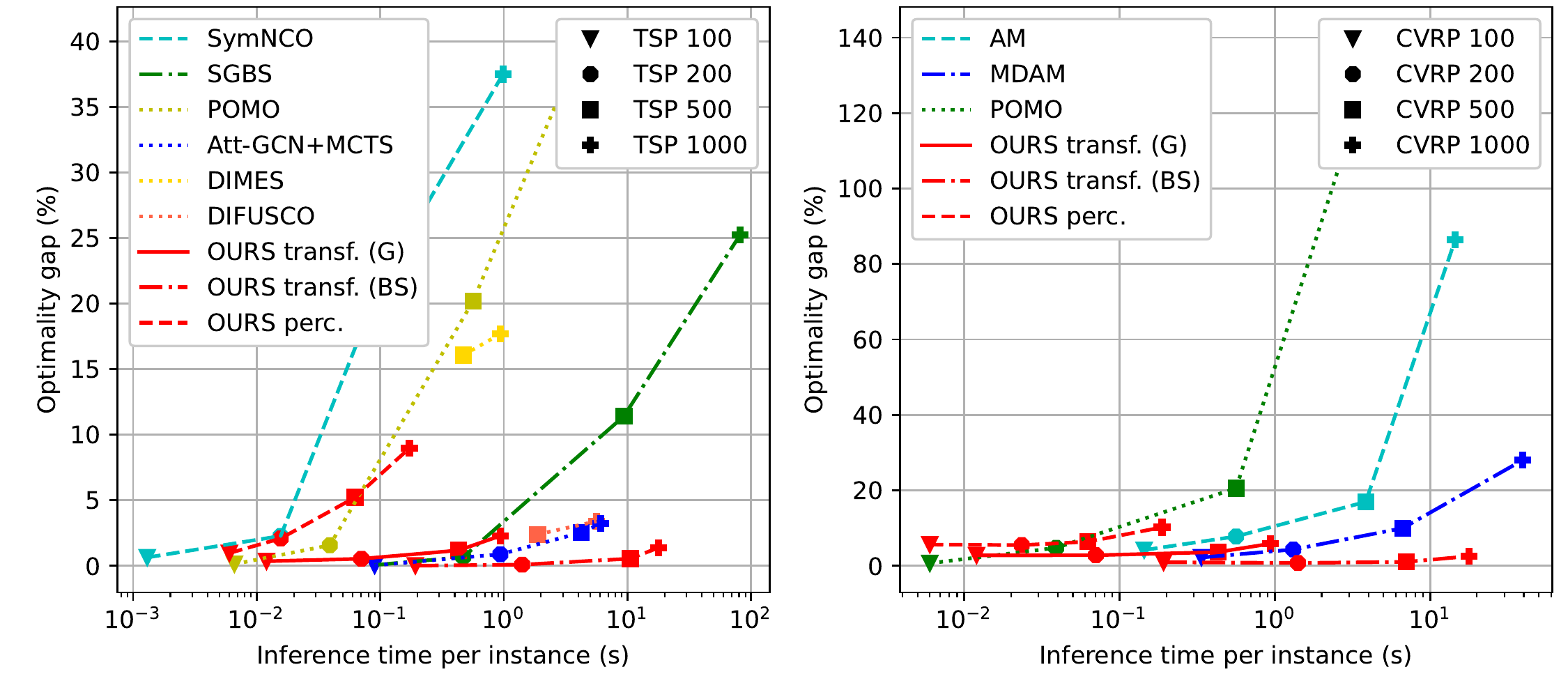}
    \caption{Generalization results on different graph sizes for TSP (left) and CVRP (right). Lower and further left is better.}
    \label{fig:optgap_time}
\end{figure*}

% interpretation
% The effectiveness of our BQ-MDP and the resulting neural architecture is confirmed by the results.
% Thanks to our principled approach, which results in better state representations and a simpler architecture, our model outperforms all baselines by generating a single solution. Additionally, we propose two other approaches based on beam search and PerceiverIO architecture, which offer different trade-offs between inference time and solution quality.

\begin{table}[h]
\centering
\resizebox{\textwidth}{!}{%

\begin{tabular}{c l | r r | r r | r r | r r}
%%%%%%% TSP %%%%%%%%%
& & \multicolumn{2}{c|}{Test (10k inst.)} & \multicolumn{6}{c}{Generalization (128 instances)} \\ 
& & \multicolumn{2}{c|}{N{=}100} & \multicolumn{2}{c|}{N{=}200} & \multicolumn{2}{c|}{N{=}500} & \multicolumn{2}{c}{N{=}1000} \\ 
&  & opt. gap & time & opt. gap & time & opt. gap & time & opt. gap & time \\ \hline\hline
\parbox[t]{2mm}{\multirow{14}{*}{\rotatebox[origin=c]{90}{TSP}}}
& Concorde & 0.00\% & 38m & 0.00\% & 2m & 0.00\% & 40m & 0.00\% & 2.5h \\
& OR-Tools & 3.76\% & 1.1h & 4.52\% & 4m & 4.89\% & 31m & 5.02\% & 2.4h \\ \cline{2-10}
& Att-GCN+MCTS$^*$ & 0.04\% & 15m & 0.88\% & 2m & 2.54\% & 6m & 3.22\% & 13m \\ 
& DIFUSCO G+2opt$^*$ & 0.24\% & - & - & - & 2.40\% & 4m & 3.40\% & 12m \\
& SGBS (10,10) & 0.06\% & 15m & 0.67\% & 1m & 11.42\% & 20m & 25.25\% & 2.9h\\
\cline{2-10}
& AM bs1024 & 2.51\% & 20m & 6.18\% & 1m & 17.98\% & 8m & 29.75\% & 31m\\
& TransTSP bs1024 &  0.46\% & 51m & 5.12\% & 1m & 36.14\% & 9m & 76.21\% & 37m\\
& MDAM bs50 &  0.39\% & 45m & 2.04\% & 3m & 9.88\% & 13m & 19.96\% & 1.1h\\
& POMO augx8 &  0.13\% & 1m & 1.57\% & 5s & 20.18\% & 1m & 40.60\% & 10m \\ 
& DIMES RL+S$^*$ & - & - & - & - & 16.07\% & 1m & 17.69\% & 2m\\
& Sym-NCO s100 & 0.64\% & 13s & 2.28\% & 2s & 21.64\% & 13s & 37.51\% & 2m\\
\cline{2-10}
& \textbf{BQ-perceiver G} & 0.97\% & 1m & 2.09\% & 2s & 5.22\% & 8s & 8.97\% & 22s\\ 
& \textbf{BQ-transformer G} & 0.35\% & 2m & 0.54\% & 9s & 1.18\% & 55s & 2.29\% & 2m\\
& \textbf{BQ-transformer bs16} & \textbf{0.01}\% & 32m & \textbf{0.09}\% & 3m & \textbf{0.55}\% & 15m & \textbf{1.38}\% & 38m\\
\hline\hline

%%%%%%% CVRP %%%%%%%%%
\parbox[t]{2mm}{\multirow{11}{*}{\rotatebox[origin=c]{90}{CVRP}}} & LKH & 0.00\% & 15.3h & 0.00\% & 30m & 0.00\% & 1.3h & 0.00\% & 2.8h \\
& HGS & -0.51\% & 15.3h & -1.02\% & 30m & -1.25\% & 1.3h & -1.10\% & 2.8h \\
& OR-Tools & 9.62\% & 1h & 10.70\% & 3m & 11.40\% & 18m & 13.56\% & 43m \\\cline{2-10}
& AM bs1024 & 4.18\% & 24m & 7.79\% & 1m & 16.96\% & 8m & 86.41\% & 31m \\
& MDAM bs50 & 2.21\% & 56m & 4.33\% & 3m & 9.99\% & 14m & 28.01\% & 1.4h \\
& POMO augx8 & 0.69\% & 1m & 4.77\% & 5s & 20.57\% & 1m & 141.06\% & 10m\\
& Sym-NCO s100 & 1.46\% & 15s & 4.84\% & 3s & 18.64\% & 13s & 119.60\% & 9m\\
& SGBS (4,4) & \textbf{0.08\%} & 15m & 2.47\% & 1m & 14.98\% & 20m & 151.07\% & 2.9h \\ \cline{2-10}
& \textbf{BQ-perceiver G}  & 5.63\% & 1m & 5.49\% & 3s & 6.39\% & 8s & 10.21\% & 24s\\ 
& \textbf{BQ-transformer G} & 2.79\% & 2m & 2.81\% & 9s & 3.64\% & 55s & 5.88\% & 2m \\
& \textbf{BQ-transformer bs16}  & 0.95\% & 32m & \textbf{0.77}\% & 3m & \textbf{1.04\%} & 15m & \textbf{2.55\%} & 38m\\
\hline\hline 
%%%%%%% OP %%%%%%
\parbox[t]{2mm}{\multirow{6}{*}{\rotatebox[origin=c]{90}{OP}}} 
& EA4OP & 0.00\% & 1.4h & 0.00\% & 3m & 0.00\% & 11m & 0.00\% & 45m \\ 
\cline{2-10}
& AM bs1024 & 2.19\% & 17m & 8.04\% & 1m & 20.73\% & 4m & 29.96\% & 15m \\
& MDAM bs50$^\dagger$ & 0.69\% & 32m & 2.98\% & 49s & 15.36\% & 4m & 26.55\% & 15m \\
& SymNCO s200$^\dagger$ & 0.45\% & 4m & 3.39\% & 2s & 15.23\% & 37s & 26.95\% & 2m\\
\cline{2-10}
& \textbf{BQ-transformer G} & 0.22\% & 2m & 1.02\% & 28s & 7.32\% & 42s & 13.12\% & 1m \\
& \textbf{BQ-transformer bs16} & \textbf{-1.13}\% & 35m & \textbf{-1.04}\% & 2m & \textbf{4.10\%} & 10m & \textbf{10.69\%} & 17m \\ \hline\hline
%%%%%%% ATSP %%%%%%%%%
\parbox[t]{2mm}{\multirow{4}{*}{\rotatebox[origin=c]{90}{ATSP}}}
& LKH & 0.00\% & 18m & 0.00\% & 29s & 0.00\% & 2.3m & 0.00\% & 9.4m \\
\cline{2-10}
& MatNet s128$\dagger$& \textbf{0.93}\% & 45m & 124.20\% & 4m & - & - & - \\ 
\cline{2-10}
& \textbf{BQ-GCN-transformer G} & 1.27\% & 1m & 1.56\% & 4s & 4.22\% & 12s & 11.59\% & 14s\\ 
& \textbf{BQ-GCN-transformer bs16} & 0.96\% & 19m & \textbf{1.41}\% & 1m & \textbf{2.43}\% & 3m & \textbf{8.26}\% & 7m\\ 
%\hline\hline
\end{tabular}} \vspace{0.1cm}
\caption{\label{tab:results}BQ-models with greedy rollouts (G) or Beam Search (bs) versus classic and neural baselines. The results for models$^*$ are taken from the original papers. For a fair comparison, models$^\dagger$ are evaluated on instances generated from the same distribution as the training datasets for those models, which may differ from our training and/or test distribution.}
\end{table}

\begin{table}[h]
\resizebox{\textwidth}{!}{%
\begin{tabular}{lrrrr}
& \multicolumn{1}{c}{MDAM} & \multicolumn{1}{c}{POMO} & \multicolumn{2}{c}{BQ (ours)}\\
\multicolumn{1}{c}{Size} & \multicolumn{1}{c}{bs50} & \multicolumn{1}{c}{x8} & \multicolumn{1}{c}{greedy} & \multicolumn{1}{c}{bs16}\\
\hline
$<$100 & 3.06\% & 0.42\% & 0.34\% & \textbf{0.06}\% \\
100-200 & 5.14\% & 2.31\% &1.99\% & \textbf{1.21}\% \\
200-500 & 11.32\% & 13.32\% &2.23\% & \textbf{0.92}\% \\
500-1K & 20.40\% & 31.58\% &2.61\% & \textbf{1.91}\% \\
$>$1K & 40.81\% & 62.61\% &6.42\% & \textbf{5.90}\% \\ \hline
All & 19.01\% & 26.30\% & 3.30\% & \textbf{2.55\%}
\end{tabular}
\hspace{0.5cm}
\begin{tabular}{lrrrr}
& \multicolumn{1}{c}{MDAM} & \multicolumn{1}{c}{POMO} & \multicolumn{2}{c}{BQ (ours)}\\
\multicolumn{1}{c}{Set (size)} & \multicolumn{1}{c}{bs50} & \multicolumn{1}{c}{augx8} & \multicolumn{1}{c}{greedy} & \multicolumn{1}{c}{bs16}\\
\hline
A (32-80) & 6.17\% & 4.86\% & 4.52\% & \textbf{1.18}\% \\
B (30-77) & 8.77\% & 5.13\% & 6.36\% & \textbf{2.48}\% \\
F (44-134) & 16.96\% & 15.49\% & 6.63\% & \textbf{5.94}\% \\
M (100-200) & 5.92\% & 4.99\% & 3.58\% & \textbf{2.66}\% \\
P (15-100) & 8.44\% & 14.69\% & 3.61\% & \textbf{1.54}\% \\
X (100-1K) & 34.17\% & 21.62\% & 9.94\% & \textbf{7.22}\% \\ \hline
All (15-1K) & 22.36\% & 15.58\% & 7.62\% & \textbf{4.83}\% \\
\end{tabular}}
\caption{\label{tab:tsplib_cvrplib}Experimental results on TSPLib (left) and CVRPLib (right).}
\end{table}

\section{Conclusion} %
We proposed a principled framework to formulate COPs as efficient symmetry-invariant MDPs. %called bisimulation-quotiented MDPs. 
First, we presented how to derive a direct MDP for any COP that is amenable to a constructive approach, and proved the exact correspondence between the optimal policies of this direct MDP and the optimal solutions of the COP it stems from. 
% Second, for the widespread class of recursive COPs, we introduced a bisimulation function that leverages the symmetries in the partial solutions  to reduce the MDP's state space via bisimulation quotienting, which effectively accounts for symmetries in the space of partial solutions.
Second, for the widespread class of recursive COPs, we introduced bisimulation quotienting to leverage the symmetries of the partial solutions. 
%This consists in representing as a same state all the partial solutions that lead to a common tail subproblem, which effectively reduces the state space and facilitates policy learning for the MDP. 
We showed that the resulting BQ-MDP has a significantly reduced state space while its optimal policies still exactly solve the original COP.
We then designed a simple Transformer-based architecture as an examplification of a possible policy network for the BQ-MDPs, and applied it to five classical COP problems: (A)TSP, CVRP, OP and KP. 
%This architecture applies with minor variations to three common COP problems - TSP, CVRP, and KP - and includes key inductive biases that are well-suited for the BQ state representation. 
%We trained this policy by imitation learning on expert trajectories, and evaluated it on instances of varying sizes and distributions. 
The resulting policy generalizes particularly well, and significantly outperforms not only similar end-to-end neural-based constructive solvers but also hybrid methods that combine learning with search. 
A limitation of our approach, shared with all constructive methods, is that it can only handle problems for which feasibility is easy to ensure throughout the construction process. In order to address problems with more challenging global constraints, approaches which combine learning and constraint programming (e.g. \cite{cappart_combining_2020}) may be a promising direction.
While training on relatively small instances allowed us to use imitation learning, our proposed MDPs and the associated model could also be trained with Reinforcement Learning. This would be particularly interesting to extend our framework from deterministic to stochastic COPs.

\clearpage

\bibliographystyle{plainnat}
\bibliography{main}

\begin{thebibliography}{51}
\providecommand{\natexlab}[1]{#1}
\providecommand{\url}[1]{\texttt{#1}}
\expandafter\ifx\csname urlstyle\endcsname\relax
  \providecommand{\doi}[1]{doi: #1}\else
  \providecommand{\doi}{doi: \begingroup \urlstyle{rm}\Url}\fi

\bibitem[Applegate et~al.(2015)Applegate, Bixby, Chvatal, and
  Cook]{applegate_concorde_2015}
David Applegate, Robert Bixby, Vasek Chvatal, and William Cook.
\newblock Concorde {{TSP}} solver.
\newblock University of Waterloo, 2015.

\bibitem[Bachlechner et~al.(2021)Bachlechner, Majumder, Mao, Cottrell, and
  McAuley]{bachlechner_rezero_2021}
Thomas Bachlechner, Bodhisattwa~Prasad Majumder, Henry Mao, Gary Cottrell, and
  Julian McAuley.
\newblock {{ReZero}} is all you need: Fast convergence at large depth.
\newblock In \emph{Proceedings of the {{Thirty-Seventh Conference}} on
  {{Uncertainty}} in {{Artificial Intelligence}}}, pages 1352--1361. {PMLR},
  December 2021.

\bibitem[Bellman(1954)]{bellman_theory_1954}
Richard Bellman.
\newblock The {{Theory}} of {{Dynamic Programming}}.
\newblock Technical Report P-550, {The Rand Corporation}, {Santa Monica, CA,
  U.S.A.}, July 1954.

\bibitem[Bello et~al.(2017)Bello, Pham, Le, Norouzi, and
  Bengio]{bello_neural_2017}
Irwan Bello, Hieu Pham, Quoc~V. Le, Mohammad Norouzi, and Samy Bengio.
\newblock Neural {{Combinatorial Optimization}} with {{Reinforcement
  Learning}}.
\newblock \emph{arXiv:1611.09940 [cs, stat]}, January 2017.

\bibitem[Bengio et~al.(2021)Bengio, Lodi, and Prouvost]{bengio_machine_2021}
Yoshua Bengio, Andrea Lodi, and Antoine Prouvost.
\newblock Machine learning for combinatorial optimization: {{A}} methodological
  tour d'horizon.
\newblock \emph{European Journal of Operational Research}, 290\penalty0
  (2):\penalty0 405--421, 2021.

\bibitem[Bertsekas(2012)]{bertsekas_dynamic_vol1_2012}
Dimitri Bertsekas.
\newblock \emph{Dynamic {{Programming}} and {{Optimal Control}}: {{Volume I}}}.
\newblock {Athena Scientific}, 2012.
\newblock ISBN 978-1-886529-43-4.

\bibitem[Bresson and Laurent(2021)]{bresson_transformer_2021}
Xavier Bresson and Thomas Laurent.
\newblock The {{Transformer Network}} for the {{Traveling Salesman Problem}}.
\newblock \emph{arXiv:2103.03012 [cs]}, March 2021.

\bibitem[Cappart et~al.(2020)Cappart, Moisan, Rousseau, {Pr{\'e}mont-Schwarz},
  and Cire]{cappart_combining_2020}
Quentin Cappart, Thierry Moisan, Louis-Martin Rousseau, Isabeau
  {Pr{\'e}mont-Schwarz}, and Andre Cire.
\newblock Combining {{Reinforcement Learning}} and {{Constraint Programming}}
  for {{Combinatorial Optimization}}.
\newblock \emph{arXiv:2006.01610 [cs]}, June 2020.

\bibitem[Cappart et~al.(2021)Cappart, Ch{\'e}telat, Khalil, Lodi, Morris, and
  Veli{\v c}kovi{\'c}]{cappart_combinatorial_2021}
Quentin Cappart, Didier Ch{\'e}telat, Elias~B. Khalil, Andrea Lodi, Christopher
  Morris, and Petar Veli{\v c}kovi{\'c}.
\newblock Combinatorial {{Optimization}} and {{Reasoning}} with {{Graph Neural
  Networks}}.
\newblock In \emph{Twenty-{{Ninth International Joint Conference}} on
  {{Artificial Intelligence}}}, volume~5, pages 4348--4355, August 2021.
\newblock \doi{10.24963/ijcai.2021/595}.

\bibitem[Choo et~al.(2022)Choo, Kwon, Kim, Jae, Hottung, Tierney, and
  Gwon]{choo_simulation-guided_2022}
Jinho Choo, Yeong-Dae Kwon, Jihoon Kim, Jeongwoo Jae, Andre Hottung, Kevin
  Tierney, and Youngjune Gwon.
\newblock Simulation-guided {{Beam Search}} for {{Neural Combinatorial
  Optimization}}.
\newblock In \emph{{{NeurIPS}} 2022}, 2022.

\bibitem[Conforti et~al.(2014)Conforti, Cornu{\'e}jols, and
  Zambelli]{conforti_integer_2014}
Michele Conforti, G{\'e}rard Cornu{\'e}jols, and Giacomo Zambelli.
\newblock \emph{Integer {{Programming}}}, volume 271 of \emph{Graduate
  {{Texts}} in {{Mathematics}}}.
\newblock {Springer International Publishing}, {Cham}, 2014.
\newblock ISBN 978-3-319-11007-3 978-3-319-11008-0.
\newblock \doi{10.1007/978-3-319-11008-0}.

\bibitem[Cook et~al.(1997)Cook, Cunningham, Pulleyblank, and
  Schrijver]{cook_combinatorial_1997}
William~J. Cook, William~H. Cunningham, William~R. Pulleyblank, and Alexander
  Schrijver.
\newblock \emph{Combinatorial {{Optimization}}}.
\newblock {Wiley-Interscience}, {New York}, 1st edition edition, November 1997.
\newblock ISBN 978-0-471-55894-1.

\bibitem[Drori et~al.(2020)Drori, Kharkar, Sickinger, Kates, Ma, Ge, Dolev,
  Dietrich, Williamson, and Udell]{drori_learning_2020}
Iddo Drori, Anant Kharkar, William~R. Sickinger, Brandon Kates, Qiang Ma, Suwen
  Ge, Eden Dolev, Brenda Dietrich, David~P. Williamson, and Madeleine Udell.
\newblock Learning to {{Solve Combinatorial Optimization Problems}} on
  {{Real-World Graphs}} in {{Linear Time}}, June 2020.

\bibitem[Fu et~al.(2021)Fu, Qiu, and Zha]{fu_generalize_2021}
Zhang-Hua Fu, Kai-Bin Qiu, and Hongyuan Zha.
\newblock Generalize a {{Small Pre-trained Model}} to {{Arbitrarily Large TSP
  Instances}}.
\newblock \emph{Proceedings of the AAAI Conference on Artificial Intelligence},
  35\penalty0 (8):\penalty0 7474--7482, May 2021.
\newblock ISSN 2374-3468.
\newblock \doi{10.1609/aaai.v35i8.16916}.

\bibitem[Geisler et~al.(2022)Geisler, Sommer, Schuchardt, Bojchevski, and
  G{\"u}nnemann]{geisler_generalization_2022}
Simon Geisler, Johanna Sommer, Jan Schuchardt, Aleksandar Bojchevski, and
  Stephan G{\"u}nnemann.
\newblock Generalization of {{Neural Combinatorial Solvers Through}} the
  {{Lens}} of {{Adversarial Robustness}}.
\newblock In \emph{Advances in {{Neural Information Processing Systems}}},
  December 2022.

\bibitem[Groshev et~al.(2018)Groshev, Goldstein, Tamar, Srivastava, and
  Abbeel]{groshev_learning_2018}
Edward Groshev, Maxwell Goldstein, Aviv Tamar, Siddharth Srivastava, and Pieter
  Abbeel.
\newblock Learning {{Generalized Reactive Policies Using Deep Neural
  Networks}}.
\newblock \emph{Proceedings of the International Conference on Automated
  Planning and Scheduling}, 28:\penalty0 408--416, June 2018.
\newblock ISSN 2334-0843.
\newblock \doi{10.1609/icaps.v28i1.13872}.

\bibitem[Helsgaun(2017)]{helsgaun_extension_2017}
Keld Helsgaun.
\newblock An {{Extension}} of the {{Lin-Kernighan-Helsgaun TSP Solver}} for
  {{Constrained Traveling Salesman}} and {{Vehicle Routing Problems}}.
\newblock Technical report, {Roskilde University}, {Roskilde, Denmark}, 2017.

\bibitem[Hou et~al.(2022)Hou, Yang, Su, Wang, and Deng]{hou_generalize_2022}
Qingchun Hou, Jingwei Yang, Yiqiang Su, Xiaoqing Wang, and Yuming Deng.
\newblock Generalize {{Learned Heuristics}} to {{Solve Large-scale Vehicle
  Routing Problems}} in {{Real-time}}.
\newblock In \emph{The {{Eleventh International Conference}} on {{Learning
  Representations}}}, September 2022.

\bibitem[Jaegle et~al.(2022)Jaegle, Borgeaud, Alayrac, Doersch, Ionescu, Ding,
  Koppula, Zoran, Brock, Shelhamer, Henaff, Botvinick, Zisserman, Vinyals, and
  Carreira]{jaegle_perceiver_2022}
Andrew Jaegle, Sebastian Borgeaud, Jean-Baptiste Alayrac, Carl Doersch, Catalin
  Ionescu, David Ding, Skanda Koppula, Daniel Zoran, Andrew Brock, Evan
  Shelhamer, Olivier~J. Henaff, Matthew Botvinick, Andrew Zisserman, Oriol
  Vinyals, and Joao Carreira.
\newblock Perceiver {{IO}}: {{A General Architecture}} for {{Structured
  Inputs}} \& {{Outputs}}.
\newblock In \emph{International {{Conference}} on {{Learning
  Representations}}}, January 2022.

\bibitem[Jiang et~al.(2021)Jiang, Cao, and Zhang]{jiang_solving_2021}
Yuan Jiang, Zhiguang Cao, and Jie Zhang.
\newblock Solving {{3D Bin Packing Problem}} via {{Multimodal Deep
  Reinforcement Learning}}.
\newblock In \emph{Proceedings of the 20th {{International Conference}} on
  {{Autonomous Agents}} and {{MultiAgent Systems}}}, {{AAMAS}} '21, pages
  1548--1550, {Richland, SC}, May 2021. {International Foundation for
  Autonomous Agents and Multiagent Systems}.
\newblock ISBN 978-1-4503-8307-3.

\bibitem[Joshi et~al.(2019)Joshi, Laurent, and Bresson]{joshi_efficient_2019}
Chaitanya~K. Joshi, Thomas Laurent, and Xavier Bresson.
\newblock An {{Efficient Graph Convolutional Network Technique}} for the
  {{Travelling Salesman Problem}}.
\newblock \emph{arXiv:1906.01227 [cs, stat]}, June 2019.

\bibitem[Joshi et~al.(2022)Joshi, Cappart, Rousseau, and
  Laurent]{joshi_learning_2022}
Chaitanya~K. Joshi, Quentin Cappart, Louis-Martin Rousseau, and Thomas Laurent.
\newblock Learning the travelling salesperson problem requires rethinking
  generalization.
\newblock \emph{Constraints}, 27\penalty0 (1-2):\penalty0 70--98, April 2022.
\newblock ISSN 1383-7133.
\newblock \doi{10.17863/CAM.85173}.

\bibitem[Khalil et~al.(2017)Khalil, Dai, Zhang, Dilkina, and
  Song]{khalil_learning_2017}
Elias Khalil, Hanjun Dai, Yuyu Zhang, Bistra Dilkina, and Le~Song.
\newblock Learning {{Combinatorial Optimization Algorithms}} over {{Graphs}}.
\newblock In I.~Guyon, U.~V. Luxburg, S.~Bengio, H.~Wallach, R.~Fergus,
  S.~Vishwanathan, and R.~Garnett, editors, \emph{Advances in {{Neural
  Information Processing Systems}} 30}, pages 6348--6358. {Curran Associates,
  Inc.}, 2017.

\bibitem[Kim et~al.(2022)Kim, Park, and Park]{kim_sym-nco_2022}
Minsu Kim, Junyoung Park, and Jinkyoo Park.
\newblock Sym-{{NCO}}: {{Leveraging Symmetricity}} for {{Neural Combinatorial
  Optimization}}.
\newblock In \emph{Advances in {{Neural Information Processing Systems}}}, {New
  Orleans}, December 2022.

\bibitem[Kingma and Ba(2015)]{kingma_adam_2015}
Diederik~P. Kingma and Jimmy Ba.
\newblock Adam: {{A Method}} for {{Stochastic Optimization}}.
\newblock In \emph{Proceedings of the 3rd {{International Conference}} on
  {{Learning Representations}} ({{ICLR}})}, {San Diego}, May 2015.

\bibitem[Kobeaga et~al.(2018)Kobeaga, Merino, and
  Lozano]{kobeaga_efficient_2018}
Gorka Kobeaga, Mar{\'i}a Merino, and Jose~A. Lozano.
\newblock An efficient evolutionary algorithm for the orienteering problem.
\newblock \emph{Computers \& Operations Research}, 90:\penalty0 42--59,
  February 2018.
\newblock ISSN 0305-0548.
\newblock \doi{10.1016/j.cor.2017.09.003}.

\bibitem[Kolobov et~al.(2012)Kolobov, Mausam, and Weld]{kolobov_theory_2012}
Andrey Kolobov, Mausam, and Daniel~S. Weld.
\newblock A theory of goal-oriented {{MDPs}} with dead ends.
\newblock In \emph{Proceedings of the {{Twenty-Eighth Conference}} on
  {{Uncertainty}} in {{Artificial Intelligence}}}, {{UAI}}'12, pages 438--447,
  {Arlington, Virginia, USA}, August 2012. {AUAI Press}.
\newblock ISBN 978-0-9749039-8-9.

\bibitem[Kool et~al.(2019)Kool, van Hoof, and Welling]{kool_attention_2019}
Wouter Kool, Herke van Hoof, and Max Welling.
\newblock Attention, {{Learn}} to {{Solve Routing Problems}}!
\newblock In \emph{International {{Conference}} on {{Learning
  Representations}}}, 2019.

\bibitem[Kwon et~al.(2020)Kwon, Choo, Kim, Yoon, Gwon, and Min]{kwon_pomo_2020}
Yeong-Dae Kwon, Jinho Choo, Byoungjip Kim, Iljoo Yoon, Youngjune Gwon, and
  Seungjai Min.
\newblock {{POMO}}: {{Policy Optimization}} with {{Multiple Optima}} for
  {{Reinforcement Learning}}.
\newblock In \emph{Advances in {{Neural Information Processing Systems}}},
  volume~33, pages 21188--21198. {Curran Associates, Inc.}, 2020.

\bibitem[Li et~al.(2021)Li, Yan, and Wu]{li_learning_2021}
Sirui Li, Zhongxia Yan, and Cathy Wu.
\newblock Learning to delegate for large-scale vehicle routing.
\newblock In \emph{Advances in {{Neural Information Processing Systems}} 34
  Pre-Proceedings}, 2021.

\bibitem[Liu et~al.(2022)Liu, Zhu, and Zhang]{liu_goal-conditioned_2022}
Minghuan Liu, Menghui Zhu, and Weinan Zhang.
\newblock Goal-{{Conditioned Reinforcement Learning}}: {{Problems}} and
  {{Solutions}}.
\newblock In \emph{Thirty-{{First International Joint Conference}} on
  {{Artificial Intelligence}}}, volume~6, pages 5502--5511, July 2022.
\newblock \doi{10.24963/ijcai.2022/770}.

\bibitem[Manchanda et~al.(2022)Manchanda, Michel, Drakulic, and
  Andreoli]{manchanda_generalization_2022}
Sahil Manchanda, Sofia Michel, Darko Drakulic, and Jean-Marc Andreoli.
\newblock On the {{Generalization}} of {{Neural Combinatorial Optimization
  Heuristics}}.
\newblock In \emph{European {{Conference}} on {{Machine Learning}} and
  {{Principles}} and {{Practice}} of {{Knowledge Discovery}} in {{Databases}}
  ({{ECML-PKDD}})}, {Grenoble, France}, September 2022.

\bibitem[Mazyavkina et~al.(2021)Mazyavkina, Sviridov, Ivanov, and
  Burnaev]{mazyavkina_reinforcement_2021}
Nina Mazyavkina, Sergey Sviridov, Sergei Ivanov, and Evgeny Burnaev.
\newblock Reinforcement learning for combinatorial optimization: {{A}} survey.
\newblock \emph{Computers \& Operations Research}, 134:\penalty0 105400,
  October 2021.
\newblock ISSN 0305-0548.
\newblock \doi{10.1016/j.cor.2021.105400}.

\bibitem[Nazari et~al.(2018)Nazari, Oroojlooy, Snyder, and
  Takac]{nazari_reinforcement_2018}
MohammadReza Nazari, Afshin Oroojlooy, Lawrence Snyder, and Martin Takac.
\newblock Reinforcement {{Learning}} for {{Solving}} the {{Vehicle Routing
  Problem}}.
\newblock In \emph{Advances in {{Neural Information Processing Systems}}},
  volume~31. {Curran Associates, Inc.}, 2018.

\bibitem[Nowak et~al.(2018)Nowak, Folqu{\'e}, and Bruna]{nowak_divide_2018}
Alex Nowak, David Folqu{\'e}, and Joan Bruna.
\newblock Divide and {{Conquer Networks}}.
\newblock In \emph{6-Th {{International Conference}} on {{Learning
  Representations}}}, {Vancouver, Canada}, April 2018.

\bibitem[Peng et~al.(2020)Peng, Wang, and Zhang]{peng_deep_2020}
Bo~Peng, Jiahai Wang, and Zizhen Zhang.
\newblock A {{Deep Reinforcement Learning Algorithm Using Dynamic Attention
  Model}} for {{Vehicle Routing Problems}}.
\newblock In Kangshun Li, Wei Li, Hui Wang, and Yong Liu, editors,
  \emph{Artificial {{Intelligence Algorithms}} and {{Applications}}},
  Communications in {{Computer}} and {{Information Science}}, pages 636--650,
  {Singapore}, 2020. {Springer}.
\newblock ISBN 9789811555770.
\newblock \doi{10.1007/978-981-15-5577-0_51}.

\bibitem[Perron and Furnon(2022)]{perron_or-tools_2022}
Laurent Perron and Vincent Furnon.
\newblock {{OR-tools}}.
\newblock Google, 2022.

\bibitem[Qiu et~al.(2022)Qiu, Sun, and Yang]{qiu_dimes_2022}
Ruizhong Qiu, Zhiqing Sun, and Yiming Yang.
\newblock {{DIMES}}: {{A Differentiable Meta Solver}} for {{Combinatorial
  Optimization Problems}}.
\newblock In \emph{Advances in {{Neural Information Processing Systems}}},
  December 2022.

\bibitem[Rossi et~al.(2006)Rossi, {van Beek}, and Walsh]{rossi_handbook_2006}
Francesca Rossi, Peter {van Beek}, and Toby Walsh.
\newblock \emph{Handbook of {{Constraint Programming}}}.
\newblock {Elsevier Science Inc.}, {USA}, July 2006.
\newblock ISBN 978-0-08-046380-3.

\bibitem[Schaul et~al.(2015)Schaul, Horgan, Gregor, and
  Silver]{schaul_universal_2015}
Tom Schaul, Daniel Horgan, Karol Gregor, and David Silver.
\newblock Universal {{Value Function Approximators}}.
\newblock In \emph{Proceedings of the 32nd {{International Conference}} on
  {{Machine Learning}}}, pages 1312--1320. {PMLR}, June 2015.

\bibitem[Son et~al.(2023)Son, Kim, Kim, and Park]{son_meta-sage_2023}
Jiwoo Son, Minsu Kim, Hyeonah Kim, and Jinkyoo Park.
\newblock Meta-{{SAGE}}: {{Scale Meta-Learning Scheduled Adaptation}} with
  {{Guided Exploration}} for {{Mitigating Scale Shift}} on {{Combinatorial
  Optimization}}, June 2023.

\bibitem[Sun and Yang(2023)]{sun_difusco_2023}
Zhiqing Sun and Yiming Yang.
\newblock {{DIFUSCO}}: {{Graph-based Diffusion Solvers}} for {{Combinatorial
  Optimization}}, February 2023.

\bibitem[Vaswani et~al.(2017)Vaswani, Shazeer, Parmar, Uszkoreit, Jones, Gomez,
  Kaiser, and Polosukhin]{vaswani_attention_2017}
Ashish Vaswani, Noam Shazeer, Niki Parmar, Jakob Uszkoreit, Llion Jones,
  Aidan~N Gomez, {\L}ukasz Kaiser, and Illia Polosukhin.
\newblock Attention is {{All}} you {{Need}}.
\newblock In \emph{Advances in {{Neural Information Processing Systems}}},
  volume~30. {Curran Associates, Inc.}, 2017.

\bibitem[Vidal(2022)]{vidal_hybrid_2022}
Thibaut Vidal.
\newblock Hybrid genetic search for the {{CVRP}}: {{Open-source}}
  implementation and {{SWAP}}* neighborhood.
\newblock \emph{Computers \& Operations Research}, 140:\penalty0 105643, April
  2022.
\newblock ISSN 0305-0548.
\newblock \doi{10.1016/j.cor.2021.105643}.

\bibitem[Vinyals et~al.(2015)Vinyals, Fortunato, and
  Jaitly]{vinyals_pointer_2015}
Oriol Vinyals, Meire Fortunato, and Navdeep Jaitly.
\newblock Pointer {{Networks}}.
\newblock In C.~Cortes, N.~D. Lawrence, D.~D. Lee, M.~Sugiyama, and R.~Garnett,
  editors, \emph{Advances in {{Neural Information Processing Systems}} 28},
  pages 2692--2700. {Curran Associates, Inc.}, 2015.

\bibitem[Xin et~al.(2021{\natexlab{a}})Xin, Song, Cao, and
  Zhang]{xin_multi-decoder_2021}
Liang Xin, Wen Song, Zhiguang Cao, and Jie Zhang.
\newblock Multi-{{Decoder Attention Model}} with {{Embedding Glimpse}} for
  {{Solving Vehicle Routing Problems}}.
\newblock \emph{Proceedings of the AAAI Conference on Artificial Intelligence},
  35\penalty0 (13):\penalty0 12042--12049, May 2021{\natexlab{a}}.
\newblock ISSN 2374-3468.
\newblock \doi{10.1609/aaai.v35i13.17430}.

\bibitem[Xin et~al.(2021{\natexlab{b}})Xin, Song, Cao, and
  Zhang]{xin_step-wise_2021}
Liang Xin, Wen Song, Zhiguang Cao, and Jie Zhang.
\newblock Step-{{Wise Deep Learning Models}} for {{Solving Routing Problems}}.
\newblock \emph{IEEE Transactions on Industrial Informatics}, 17\penalty0
  (7):\penalty0 4861--4871, July 2021{\natexlab{b}}.
\newblock ISSN 1941-0050.
\newblock \doi{10.1109/TII.2020.3031409}.

\bibitem[Xin et~al.(2022)Xin, Song, Cao, and Zhang]{xin_generative_2022}
Liang Xin, Wen Song, Zhiguang Cao, and Jie Zhang.
\newblock Generative {{Adversarial Training}} for {{Neural Combinatorial
  Optimization Models}}.
\newblock January 2022.

\bibitem[Zhang et~al.(2020)Zhang, Song, Cao, Zhang, Tan, and
  Xu]{zhang_learning_2020}
Cong Zhang, Wen Song, Zhiguang Cao, Jie Zhang, Puay~Siew Tan, and Chi Xu.
\newblock Learning to dispatch for job shop scheduling via deep reinforcement
  learning.
\newblock In \emph{Proceedings of the 34th {{International Conference}} on
  {{Neural Information Processing Systems}}}, {{NIPS}}'20, pages 1621--1632,
  {Red Hook, NY, USA}, December 2020. {Curran Associates Inc.}
\newblock ISBN 978-1-71382-954-6.

\bibitem[Zhou et~al.(2023)Zhou, Wu, Song, Cao, and Zhang]{zhou_towards_2023}
Jianan Zhou, Yaoxin Wu, Wen Song, Zhiguang Cao, and Jie Zhang.
\newblock Towards {{Omni-generalizable Neural Methods}} for {{Vehicle Routing
  Problems}}, June 2023.

\bibitem[Zong et~al.(2022)Zong, Wang, Wang, Zheng, and Li]{zong_rbg_2022}
Zefang Zong, Hansen Wang, Jingwei Wang, Meng Zheng, and Yong Li.
\newblock {{RBG}}: {{Hierarchically Solving Large-Scale Routing Problems}} in
  {{Logistic Systems}} via {{Reinforcement Learning}}.
\newblock In \emph{Proceedings of the 28th {{ACM SIGKDD Conference}} on
  {{Knowledge Discovery}} and {{Data Mining}}}, pages 4648--4658, {Washington
  DC USA}, August 2022. {ACM}.
\newblock ISBN 978-1-4503-9385-0.
\newblock \doi{10.1145/3534678.3539037}.

\end{thebibliography}

\clearpage
\appendix
\section{Application to the Capacitated Vehicle Routing Problem}
\label{anx:cvrp}
%%%%%%%%%%%%%%%%%%%%%%%%%%%%%%%%%%%%%%%%%%%%%%%%%%%%%%%%%%%%%%%%%%%%%%%%%%%
\paragraph{Problem definition.}
%%%%%%%%%%%%%%%%%%%%%%%%%%%%%%%%%%%%%%%%%%%%%%%%%%%%%%%%%%%%%%%%%%%%%%%%%%%
The Capacitated Vehicle Routing Problem (CVRP) is a vehicle routing problem in which a vehicle with limited \emph{capacity} must deliver items from a depot location to various customer locations. Each customer has an associated \emph{demand}, and the goal is to compute a set of subtours for the vehicle, starting and ending at the depot, such that all the customers are visited, the sum of the demands per subtour of the vehicle does not exceed the capacity, and the total travelled distance is minimized.

% returning as many times as needed to the depot to refill, but without ever exceeding the vehicle capacity.

%Formally, we assume given a set of customer nodes, each with a demand (positive scalar), plus a depot node. A CVRP solution (in $\MX$) is a finite sequence of nodes starting at the depot, which are pairwise distinct except for the depot, and respecting the capacity constraint: the total demand of any contiguous sub-sequence of customer nodes is below the vehicle capacity. A CVRP instance (in $\MF_{\MX}$) is given by a finite set $D$ of nodes, including the depot, their coordinates in the Euclidian space $V$, and maps any solution to the length of the corresponding path using the distances in $V$, if the path visits exactly all the nodes of $D$, or $\infty$ otherwise (unfeasible solutions).

\paragraph{Solution space.}
Formally, a partial solution (in $\MX$) for CVRP, just as for TSP, is a finite sequence of nodes. Similarly, the $\circ$ operator is sequence concatenation and the neutral element $\epsilon$ is the empty sequence. In a CVRP instance, as in TSP, each node is assigned a location, and the objective function $f$ for an arbitrary sequence of nodes is the total travelled distance for a vehicle visiting the nodes in sequence:
\[
f(x_{1:n}) \;=\; {\sum}_{i=2}^n\textrm{dist}(x_{i-1},x_i).
\]
The feasible set $X$ consists of the sequences $x_{1:n}$ of nodes which start and end at the depot, which are pairwise distinct except for the depot, and such that the cumulated demand of any contiguous subsequence $x_{i:j}$ not visiting the depot, i.e. a segment of a subtour, does not exceed the capacity of the vehicle:
\[
x_{1:n}\in X
\textrm{ iff }
\left\{
\begin{array}{l}
x_1=x_n=\textrm{depot},\\
\forall i,j\in\{1{:}n\}\;\;\;
x_i = x_j\not=\textrm{depot} \;\implies\; i=j,\\
\forall i,j\in\{1{:}n\}\;\;\;
\forall k\in\{i{:}j\}\;x_k\not=\textrm{depot} \;\implies\; {\sum}_{k=i}^j\textrm{demand}(x_k)\leq\textrm{capacity}.
\end{array}
\right.
\]

Just as with TSP, CVRP on its own does not satisfy the tail-recursion property of Sec.~\ref{sec:recursion}, but is a particular case of a more general problem called \textit{path}-CVRP (similar to path-TSP) which does satisfy that property. In path-CVRP, instead of starting at the depot with its full capacity, the vehicle starts at an {\em origin} node with a given {\em initial capacity}. 
% Therefore we replace the constraint $x_1{=}x_n{=}\textrm{depot}$ by $x_1{=}\textrm{origin}$ and $x_n{=}\textrm{depot}$. Furthermore, two capacities must be specified: the standard ``full'' capacity obtained after each visit at the depot, and also the available capacity at the origin. 
A CVRP instance is a path-CVRP instance where origin and depot are the same and the initial capacity is the full capacity. In path-CVRP, each tail subproblem after selection of a node $z$ updates both the origin (which becomes $z$) and the initial capacity (which is decremented by the demand at $z$ if $z$ is a customer node or reset to the full capacity if $z$ is the depot), conditioned on the resulting capacity being non negative.

\paragraph{Bisimulation quotienting.}
Given the above solution space, we can directly apply the definitions of Sec~\ref{sec:bq-mdp} to define the bisimulation and the resulting BQ-MDP for path-CVRP, illustrated below:

% We can define a mapping $\boldsymbol{\Phi}$ from path-CVRP-MDP states to a new ``path-CVRP'' state space, informally described by the following diagram.
\begin{center}
\begin{tikzpicture}[scale=2,n/.style={circle,fill,blue,inner sep=1},nn/.style={circle,fill,black,inner sep=1}]
\begin{scope}[local bounding box=a]
\node at (-.74,-.51) {}; \node at (.89,.71) {};
\node[circle,draw,inner sep=1,label={[right,orange]\tiny $C{=}12$}] (d) at (.07,.12) {};
\node[isosceles triangle,fill,green,inner sep=1,label={[below=1pt, right=-1pt,orange]\tiny $c{=}7$}] (a0) at (.16,.38) {};
\node[nn,label={[above=1pt, left,orange]\tiny 1}] (a1) at (-.12,.54) {};
\node[nn,label={[above=1pt, left,orange]\tiny 4}] (a2) at (-.37,.45) {};
\node[nn,label={[below=3pt, left,orange]\tiny 3}] (a3) at (-.17,.21) {};
\node[nn,label={[above=1pt, left,orange]\tiny 2}] (a4) at (-.66,-.51) {};
\node[nn,label={[above=1pt, left,orange]\tiny 4}] (a5) at (-.04,-.2) {};
\node[nn,label={[above=1pt, right=-1pt,orange]\tiny 1}] (a6) at (.50,-.46) {};
\draw[->,red] (a0) -- (a1) -- (a2) -- (d) -- (a3) -- (a4) -- (a5) -- (a6);
\node[n,label={[above=1pt, left,orange]\tiny 1}] at (.89,.08) {};
\node[n,label={[above=1pt, left,orange]\tiny 4}] at (.55,.49) {};
\node[n,label={[above=1pt, left,orange]\tiny 3}] at (.69,.71) {};
\draw ($(a.north east)+(.03,.03)$) rectangle ($(a.south west)+(-.03,-.03)$);
\node[anchor=north west] at ($(a.north west)$) {\tiny Direct MDP state};
\end{scope}
\begin{scope}[xshift=3cm,local bounding box=b]
\node at (-.74,-.51) {}; \node at (.89,.71) {};
\node[circle,draw,inner sep=1,label={[left,orange]\tiny $C{=}12$}] at (.07,.12) {};
\node[isosceles triangle,fill,green,inner sep=1,label={[below=1pt, right=-1pt,orange]\tiny $c{=}2$}] (b6) at (.50,-.46) {};
\node[n,label={[above=1pt, left,orange]\tiny 1}] at (.89,.08) {};
\node[n,label={[above=1pt, left,orange]\tiny 4}] at (.55,.49) {};
\node[n,label={[above=1pt, left,orange]\tiny 3}] at (.69,.71) {};
\draw ($(b.north east)+(.03,.03)$) rectangle ($(b.south west)+(-.03,-.03)$);
\node[anchor=north west] at ($(b.north west)$) {\tiny BQ-MDP state};
\end{scope}
\draw[->] (a.east) -- (b.west) node [midway,above] {$\boldsymbol{\Phi}$};
\end{tikzpicture}
\end{center}
The diagram on the left represents a direct MDP state, i.e. a path-CVRP instance together with a partial solution. The depot is represented as a white disk with its full capacity ($C{=}12$), the origin as green triangle with initial capacity ($c{=}7$), customer nodes as blue or black disks with their demands. The partial solution is the sequence of black nodes represented by the (directed) red path. The diagram on the right represents the corresponding BQ-MDP state, i.e. a path-CVRP instance. Note the new origin (the end node of the partial solution), and its initial capacity $c{=}2$: it is the full capacity $C{=}12$ minus the cumulated demand served since the last visit to the depot ($3{+}2{+}4{+}1$)

% Here, the capacity of the vehicle is $C{=}10$, shown next to the (colourless) depot node, and the demand of each node is shown next to it, in orange. The black dotted line indicates that the action which introduced the node with demand $2$ was via the depot: its flag was set to $1$ (all the other actions had their flag set to $0$ in this simple example). The green dotted line indicates how the path is closed to measure its length. After the node with demand $2$, the path of visited nodes (in red) continues with nodes with demand $4$ and $1$, respectively, so that the remaining capacity at the end of the path is $C{-}(2{+}4{+}1){=}3$. Compared to TSP, this is the additional piece of information in the summary of the ``past'' (path of visited nodes) which is preserved in the path-CVRP state, together with the origin and destination of the path. Mapping $\boldsymbol{\Phi}$ induces a bisimulation on CVRP-MDP states, and by quotienting, one obtains an MDP which can be defined directly on path-CVRP states.
%%%%%%%%%%%%%%%%%%%%%%%%%%%%%%%%%%%%%%%%%%%%%%%%%%%%%%%%%%%%%%%%%%%%%%%%%%%
\paragraph{Model architecture.}
%%%%%%%%%%%%%%%%%%%%%%%%%%%%%%%%%%%%%%%%%%%%%%%%%%%%%%%%%%%%%%%%%%%%%%%%%%%
The model architecture for the CVRP is almost the same as for the TSP, with a slight difference in the input and output layers. In the TSP model, the input to the node embedding layer for a $N$-node state is a $2{\times}N$ matrix of coordinates. For CVRP, we use two additional channels: one for the node's demand, and one for the current vehicle capacity, repeated across all nodes. The demand is set to zero for the origin and depot nodes. We obtain a $4{\times}N$ matrix of features, which is passed through a learned embedding layer. As for the TSP, a learned origin (resp. depot) encoding vector is added to the corresponding node embeddings. The rest of the architecture, in the form of attention layers, is identical to TSP, until after the action scores projection layer. In the case of TSP, the projection layer returns a vector of $N$ scores, where each score, after a softmax, represents the probability of choosing that node as the next step in the construction. In the case of CVRP, the model returns a matrix of scores of dimension $N {\times} 2$, corresponding to each possible actions and the softmax scopes over this whole matrix. An action is here either the choice of the next node, as in TSP, or of the next two nodes, the first one being the depot %(in the solution space, the steps are in fact lists of length 1 or 2). 
As usual, a mask is always applied to unfeasible actions before the softmax operator: those which have higher demand than the remaining vehicle capacity, as well as the origin and depot nodes.

%It should be noted that this architecture does not allow sampling arbitrary sequences from a full solution during the training as was the case with TSP. Indeed, each training sequence must have the end (target) in the depot. Otherwise, input sequence would have to have another channel for the 'final' vehicle capacity in the last node. This approach is also possible, but we choose previously described one, in order to keep the model as simple as possible. 

%%%%%%%%%%%%%%%%%%%%%%%%%%%%%%%%%%%%%%%%%%%%%%%%%%%%%%%%%%%%%%%%%%%%%%%%%%%
\section{Application to the Orienteering Problem}
\label{anx:op}
%%%%%%%%%%%%%%%%%%%%%%%%%%%%%%%%%%%%%%%%%%%%%%%%%%%%%%%%%%%%%%%%%%%%%%%%%%%
\paragraph{Problem definition.}
The Orienteering Problem (OP) is a combinatorial optimization problem in which we need to find the optimal route to visit a set of given locations within a given distance (or time) limit. The route must start and end at a given location (usually called a depot), each location is associated with a scalar prize, and the goal is to maximize the cumulated prize, respecting the distance (time) constraint. This problem has applications in various fields, including logistics and planning.

\paragraph{Solution space and bisimulation quotienting.} After defining BQ-MDP for path-TSP and path-CVRP, defining BQ-MDP for path-OP is straightforward. The partial solution is the sequence of already visited nodes, together with remaining distance limit. An OP instance is a path-OP instance where origin and destination are the same (depot) and distance limit equals the initial distance constraint. Just as in path-TSP and path-CVRP, in path-OP, each tail subproblem, after selecting a node $z$, updates both the origin (which becomes $z$) and the remaining distance constraint (the current distance constraint is decreased by the distance from the previous origin to the selected node $z$). However, in OP, unlike TSP and CVRP, the number of steps is not known in advance - for TSP and CVRP, construction of solution ends when all locations are visited, whereas in OP, it ends when a given distance budget is exceeded.

%%%%%%%%%%%%%%%%%%%%%%%%%%%%%%%%%%%%%%%%%%%%%%%%%%%%%%%%%%%%%%%%%%%%%%%%%%%
\paragraph{Model architecture.}
%%%%%%%%%%%%%%%%%%%%%%%%%%%%%%%%%%%%%%%%%%%%%%%%%%%%%%%%%%%%%%%%%%%%%%%%%%%
The model architecture for path-OP is the same as for path-TSP, with two additional input channels for node prize and distance constraint (repeated across all nodes, as the total capacity in CVRP). Thus, we obtain a $4{\times}N$ matrix of features (two for coordinates, one for prize and one for distance constraint), which is passed through a learned embedding layer. As usual, a learned encoding for origin and depot nodes are added to the corresponding node embeddings. The rest of the architecture, including output projection layer is the same as for TSP. Before applying the softmax operator, we apply a mask to exclude origin and depot nodes, as well as nodes that cannot be visited due to the distance constraint.

%%%%%%%%%%%%%%%%%%%%%%%%%%%%%%%%%%%%%%%%%%%%%%%%%%%%%%%%%%%%%%%%%%%%%%%%%%%
\section{Application to the Knapsack Problem}
\label{anx:kp}
%%%%%%%%%%%%%%%%%%%%%%%%%%%%%%%%%%%%%%%%%%%%%%%%%%%%%%%%%%%%%%%%%%%%%%%%%%%
The knapsack problem and its solution space were described in Sec.~\ref{sec:solution-space}; the associated bisimulation in Sec~\ref{sec:recursion}. We provide below an illustration of the bisimulation on an example:

\tikzset{every picture/.style={line width=0.75pt}} %set default line width to 0.75pt
\begin{tikzpicture}[x=0.75pt,y=0.75pt,yscale=-1,xscale=1]
%uncomment if require: \path (0,235); %set diagram left start at 0, and has height of 235

%Shape: Grid [id:dp9148158104165465] 
\draw  [draw opacity=0] (77,34) -- (298,34) -- (298,74) -- (77,74) -- cycle ; \draw   (77,34) -- (77,74)(97,34) -- (97,74)(117,34) -- (117,74)(137,34) -- (137,74)(157,34) -- (157,74)(177,34) -- (177,74)(197,34) -- (197,74)(217,34) -- (217,74)(237,34) -- (237,74)(257,34) -- (257,74)(277,34) -- (277,74)(297,34) -- (297,74) ; \draw   (77,34) -- (298,34)(77,54) -- (298,54)(77,74) -- (298,74) ; \draw    ;
%Shape: Grid [id:dp30238945439183784] 
\draw  [draw opacity=0] (374,34) -- (535,34) -- (535,75) -- (374,75) -- cycle ; \draw   (374,34) -- (374,75)(394,34) -- (394,75)(414,34) -- (414,75)(434,34) -- (434,75)(454,34) -- (454,75)(474,34) -- (474,75)(494,34) -- (494,75)(514,34) -- (514,75)(534,34) -- (534,75) ; \draw   (374,34) -- (535,34)(374,54) -- (535,54)(374,74) -- (535,74) ; \draw    ;
%Straight Lines [id:da5330434421556514] 
\draw    (297,54) -- (372,54) ;
\draw [shift={(374,54)}, rotate = 180] [color={rgb, 255:red, 0; green, 0; blue, 0 }  ][line width=0.75]    (15.3,-4.61) .. controls (9.73,-1.96) and (4.63,-0.42) .. (0,0) .. controls (4.63,0.42) and (9.73,1.96) .. (15.3,4.61)   ;

% Text Node
\draw (80,57) node [anchor=north west][inner sep=0.75pt]    {$3$};
% Text Node
\draw (100,57) node [anchor=north west][inner sep=0.75pt]  [color={rgb, 255:red, 208; green, 2; blue, 27 }  ,opacity=1 ]  {$7$};
% Text Node
\draw (120,57) node [anchor=north west][inner sep=0.75pt]    {$9$};
% Text Node
\draw (140,57) node [anchor=north west][inner sep=0.75pt]  [color={rgb, 255:red, 208; green, 2; blue, 27 }  ,opacity=1 ]  {$1$};
% Text Node
\draw (160,57) node [anchor=north west][inner sep=0.75pt]    {$1$};
% Text Node
\draw (180,57) node [anchor=north west][inner sep=0.75pt]  [color={rgb, 255:red, 208; green, 2; blue, 27 }  ,opacity=1 ]  {$2$};
% Text Node
\draw (200,57) node [anchor=north west][inner sep=0.75pt]    {$4$};
% Text Node
\draw (220,57) node [anchor=north west][inner sep=0.75pt]    {$5$};
% Text Node
\draw (240,57) node [anchor=north west][inner sep=0.75pt]    {$8$};
% Text Node
\draw (260,57) node [anchor=north west][inner sep=0.75pt]    {$8$};
% Text Node
\draw (280,57) node [anchor=north west][inner sep=0.75pt]    {$6$};
% Text Node
\draw (24,57) node [anchor=north west][inner sep=0.75pt]  [font=\small] [align=left] {weights};
% Text Node
\draw (30,36) node [anchor=north west][inner sep=0.75pt]  [font=\small] [align=left] {values};
% Text Node
\draw (80,36) node [anchor=north west][inner sep=0.75pt]    {$1$};
% Text Node
\draw (100,36) node [anchor=north west][inner sep=0.75pt]  [color={rgb, 255:red, 208; green, 2; blue, 27 }  ,opacity=1 ]  {$9$};
% Text Node
\draw (120,36) node [anchor=north west][inner sep=0.75pt]    {$2$};
% Text Node
\draw (140,36) node [anchor=north west][inner sep=0.75pt]  [color={rgb, 255:red, 208; green, 2; blue, 27 }  ,opacity=1 ]  {$8$};
% Text Node
\draw (160,36) node [anchor=north west][inner sep=0.75pt]    {$3$};
% Text Node
\draw (180,36) node [anchor=north west][inner sep=0.75pt]  [color={rgb, 255:red, 208; green, 2; blue, 27 }  ,opacity=1 ]  {$7$};
% Text Node
\draw (200,36) node [anchor=north west][inner sep=0.75pt]    {$1$};
% Text Node
\draw (220,36) node [anchor=north west][inner sep=0.75pt]    {$6$};
% Text Node
\draw (240,36) node [anchor=north west][inner sep=0.75pt]    {$7$};
% Text Node
\draw (260,36) node [anchor=north west][inner sep=0.75pt]    {$3$};
% Text Node
\draw (280,36) node [anchor=north west][inner sep=0.75pt]    {$9$};
% Text Node
\draw (75,20) node [anchor=north west][inner sep=0.75pt]   [align=left] {$C=20$};
% Text Node
\draw (380,57) node [anchor=north west][inner sep=0.75pt]    {$3$};
% Text Node
\draw (400,57) node [anchor=north west][inner sep=0.75pt]    {$9$};
% Text Node
\draw (420,57) node [anchor=north west][inner sep=0.75pt]    {$1$};
% Text Node
\draw (440,57) node [anchor=north west][inner sep=0.75pt]    {$4$};
% Text Node
\draw (460,57) node [anchor=north west][inner sep=0.75pt]    {$5$};
% Text Node
\draw (480,57) node [anchor=north west][inner sep=0.75pt]    {$8$};
% Text Node
\draw (500,57) node [anchor=north west][inner sep=0.75pt]    {$8$};
% Text Node
\draw (520,57) node [anchor=north west][inner sep=0.75pt]    {$6$};
% Text Node
\draw (380,36) node [anchor=north west][inner sep=0.75pt]    {$1$};
% Text Node
\draw (400,36) node [anchor=north west][inner sep=0.75pt]    {$2$};
% Text Node
\draw (420,36) node [anchor=north west][inner sep=0.75pt]    {$3$};
% Text Node
\draw (440,36) node [anchor=north west][inner sep=0.75pt]    {$1$};
% Text Node
\draw (460,36) node [anchor=north west][inner sep=0.75pt]    {$6$};
% Text Node
\draw (480,36) node [anchor=north west][inner sep=0.75pt]    {$7$};
% Text Node
\draw (500,36) node [anchor=north west][inner sep=0.75pt]    {$3$};
% Text Node
\draw (520,36) node [anchor=north west][inner sep=0.75pt]    {$9$};
% Text Node
\draw (372,20) node [anchor=north west][inner sep=0.75pt]   [align=left] {$C=10$};
% Text Node
\draw (157,77) node [anchor=north west][inner sep=0.75pt]   [align=left] {KP direct MDP state};
% Text Node
\draw (414,77) node [anchor=north west][inner sep=0.75pt]   [align=left] {KP BQ-MDP state};
% Text Node
\draw (331,40) node [anchor=north west][inner sep=0.75pt]    {$\boldsymbol{\Phi} $};
\end{tikzpicture}
Here, the knapsack capacity is $C=20$ and each item is defined by its weight (bottom cell) and value (top cell). Mapping $\boldsymbol{\Phi}$ for KP is straightforward: it removes all picked items and updates the remaining capacity by subtracting total weight of removed items from the previous capacity.

%%%%%%%%%%%%%%%%%%%%%%%%%%%%%%%%%%%%%%%%%%%%%%%%%%%%%%%%%%%%%%%%%%%%%%%%%%%
\paragraph{Model architecture.}
%%%%%%%%%%%%%%%%%%%%%%%%%%%%%%%%%%%%%%%%%%%%%%%%%%%%%%%%%%%%%%%%%%%%%%%%%%%
The model architecture for KP is again very similar to previously described models for TSP, CVRP and OP. The input to the model is a $3\times N$ tensor composed of items features (values, weights) and the additional channel for the remaining knapsack's capacity. By definition, the solution has no order (the result is a set of items), so there is no need to add tokens for origin and destination. Apart from excluding these tokens and different input dimensions, the rest of the model is identical to the TSP model. The output is a vector of $N$ probabilities over all items with a mask over the unfeasible ones (with weights larger than remaining knapsack's capacity). In the training, at each construction step, any item of the ground-truth solution is a valid choice. Therefore we use a multi-class cross-entropy loss.

%%%%%%%%%%%%%%%%%%%%%%%%%%%%%%%%%%%%%%%%%%%%%%%%%%%%%%%%%%%%%%%%%%%%%%%%%%%
\paragraph{Experimental results for KP.}
%%%%%%%%%%%%%%%%%%%%%%%%%%%%%%%%%%%%%%%%%%%%%%%%%%%%%%%%%%%%%%%%%%%%%%%%%%%
We generate the training dataset as described in \cite{kwon_pomo_2020}. We train our model on 1M KP instances of size 200 and capacity 25, with values and weights randomly sampled from the unit interval. We use the dynamic programming algorithm from ORTools to compute the ground-truth optimal solutions. As hyperparameters, we use the same as for the TSP, except the training is shorter - it converges in just 50 epochs. Then, we evaluate our model on test datasets with 200, 500 and 1000 items and capacity of 10, 25, 50 and 100, for each problem size. Table \ref{tab:KP} shows the performance of our model compared to POMO, one of the best performing NCO models on KP. Although our model does not outperform it on all datasets, it achieves better overall performance and significantly better performance on the out-of-distribution datasets (datasets of size 1000 and datasets with a capacity of 10). It should be noted again that POMO builds $N$ solutions per instance and chooses the best one, while our model generates a single solution per instance but still achieves better results. 

\begin{table*}
    \centering
    \begin{tabular}{l l | r | r r | r r | r r}
          \multicolumn{2}{c|}{} & \multicolumn{1}{c|}{Optimal} & \multicolumn{2}{c|}{POMO (single traj.)} & \multicolumn{2}{c|}{POMO (all traj.)} & \multicolumn{2}{c}{BQ \textbf{ours} (greedy)} \\
          \multicolumn{2}{c|}{} & \multicolumn{1}{c|}{value} & \multicolumn{1}{c}{value} & \multicolumn{1}{c|}{optgap} & \multicolumn{1}{c}{value} & \multicolumn{1}{c|}{optgap} & \multicolumn{1}{c}{value} & \multicolumn{1}{c}{optgap}  \\
          \hline \hline
          \multirow{4}{*}{N=200} & C=10 & 36.073 & 34.062 & 5.565\%& 34.961& 3.076\%& \textbf{35.961}& \textbf{0.311}\%\\
           & C=25 & 57.429 & 57.143& 0.499\%& \textbf{57.420}& \textbf{0.016}\%& 57.371& 0.102\%\\
           & C=50 & 81.100& 79.766& 1.617\%& 80.085 & 1.229\%& \textbf{80.564}& \textbf{0.668}\%\\
           & C=100 & 99.773& 99.416& 0.358\%& 99.483& 0.291\%& \textbf{99.694} & \textbf{0.080}\%\\ \hline
          \multirow{4}{*}{N=500} & C=10 & 57.456& 51.829& 9.769\%& 54.213& 5.627\%& \textbf{56.853}& \textbf{1.054}\%\\
           & C=25 & 91.026 & 85.186 & 6.414\%& 86.482& 4.992\%& \textbf{90.741}& \textbf{0.314}\%\\
           & C=50 & 128.999 & 128.646 & 0.273\%& \textbf{128.946} & \textbf{0.042}\%& 128.906 & 0.072\%\\
           & C=100 & 182.395 & 181.615 & 0.424\%& \textbf{181.870} & \textbf{0.285}\%& 181.654 & 0.407\%\\ \hline
          \multirow{4}{*}{N=1000} & C=10 & 81.334 & 53.319 & 34.401\%& 58.072 & 28.565\%& \textbf{79.650}& \textbf{2.074}\%\\
           & C=25 & 128.993 & 122.112 & 5.340\% & 123.775& 4.046\% & \textbf{128.240} & \textbf{0.584}\%\\
           & C=50 & 182.813 & 170.223 & 6.877\%& 171.789 & 6.021\%& \textbf{181.985}& \textbf{0.451}\%\\
           & C=100 & 257.411 & 252.701 & 1.831\%& 253.361 & 1.575\%& \textbf{257.224} & \textbf{0.072}\%\\ \hline\hline
          \multicolumn{2}{c|}{All} & - & & 6.131\% & & 4.647\% & & \textbf{0.516}\%  
    \end{tabular}
    \caption{Average values and optimality gaps for KP on various instance distributions}
    \label{tab:KP}
\end{table*}

%%%%%%%%%%%%%%%%%%%%%%%%%%%%%%%%%%%%%%%%%%%%%%%%%%%%%%%%%%%%%%%%%%%%%%%%%%
\section{Impact of k-nearest-neighbor heuristic on model performance}
\label{anx:knns}
%%%%%%%%%%%%%%%%%%%%%%%%%%%%%%%%%%%%%%%%%%%%%%%%%%%%%%%%%%%%%%%%%%%%%%%%%%%
As mentioned in the Sec.~\ref{sec:experiments}, inference time of our model can be reduced by using a $k$-nearest-neighbor heuristic to restrict the search space. For both greedy rollouts and beam search strategies, at every step, it is possible to reduce the remaining graph by considering only a certain number of neighboring nodes. Table \ref{tab:knns} presents the experiments on TSP and CVRP where we apply the model on different number of nearest neighbors of the origin. This approach clearly reduces the execution time, but also in some cases even improves the performance in terms of optimality gap.
Note that the criteria on which to select the nearest neighbors does not have to be the distance but the same metric as some greedy heuristic for the problem. For example for the Knapsack problem, the items could be restricted to the $k$ items with highest values (or highest ratios of value/weight).

\begin{table*}
    \centering
    \begin{tabular}{l l | r r  r r | r r  r r }
          & &  \multicolumn{4}{c|}{TSP} & \multicolumn{4}{c}{CVRP} \\
          & & \multicolumn{2}{c}{Greedy} & \multicolumn{2}{c|}{Beam size 16} & \multicolumn{2}{c} {Greedy} & \multicolumn{2}{c} {Beam size 16} \\ \hline
          N=500 & full graph & 1.091\% & 2m & 0.572\%& 28m & 3.951\%
& 2m & 1.503\% & 28m\\ 
           & 250 KNNs & 1.186\% & 1m & 0.550\% & 15m & 3.645\% & 1m & 1.040\% & 15m \\ \hline
          N=1000 & full graph & 2.141\% & 14m & 1.412\% & 3.5h & 6.282\% &
15m & 3.660\% & 5.4h \\
          & 500 KNNs & 2.086\% & 7m & 1.348\% & 2.7h & 6.330\% & 10m & 3.382\% & 2.7h \\ 
           & 250 KNNs & 2.294\% & 2m &1.379\% & 38m & 5.883\% & 2m & 2.552\% & 38m \\

    \end{tabular}
    \caption{Experimental results with different numbers $N$ of nearest-neighbors during the inference. }
    \label{tab:knns}
\end{table*}
%%%%%%%%%%%%%%%%%%%%%%%%%%%%%%%%%%%%%%%%%%%%%%%%%%%%%%%%%%%%%%%%%%%%%%%%%%%
\section{Ablation study}
%%%%%%%%%%%%%%%%%%%%%%%%%%%%%%%%%%%%%%%%%%%%%%%%%%%%%%%%%%%%%%%%%%%%%%%%%%%

%%%%%%%%%%%%%%%%%%%%%%%%%%%%%%%%%%%%%%%%%%%%%%%%%%%%%%%%%%%%%%%%%%%%%%%%%%%
\subsection{PerceiverIO architecture}
\label{anx:perceiver}

To construct a solution, our model needs to perform $N$ steps, and compute $N^2$ attention matrices at each step, so total complexity is $\mathcal{O}(N^3)$. Although our model outperforms current state-of-the-art models in terms of both performance and inference time, this may become a limiting factor when applying the model to large graph sizes. This is a well-known issue for all attention models, and there have been various proposals to reduce the complexity of attention. In this work, we propose a compromise between model complexity and quality of the solution by replacing the standard transformer model with the PerceiverIO architecture \cite{jaegle_perceiver_2022}. PerceiverIO computes cross-attention between input data and latent variables and then compute self-attention between the latent variables,  resulting in all computations being done in the latent space. This approach allows the number of operations to be linear (instead of quadratic) in the input's length. 

In our implementation, we use similar hyperparameters as for the transformer model: 9 attention layers with 8 heads, an embedding size of 192, and a feed-forward layer dimension of 512. For the latent space, we use a vector with dimensions of $64{\times}48$, while the output query array is the same as the input array.

%%%%%%%%%%%%%%%%%%%%%%%%%%%%%%%%%%%%%%%%%%%%%%%%%%%%%%%%%%%%%%%%%%%%%%%%%%%

\subsection{Approximated model}

As mentioned in Section ~\ref{sec:related-work}, existing works have also noted the importance of accounting for the change of the state after each action: \cite{xin_step-wise_2021, xin_multi-decoder_2021} claimed that models should recompute the embeddings after each action. 
However because of the additional training cost, they proposed the following approximation: fixing lower encoder levels and recomputing just the top level with a mask of already visited nodes. 
They hypothesis a kind of hierarchical feature extraction property that may make the last layers more important for the fine-grained next decision. 
In contrast, we call our entire model after each construction step; effectively recomputing the embeddings of each state. 
We hypothesis that this property may explain the superior performance (Table~1) w.r.t MDAM model \cite{xin_multi-decoder_2021}. 
In order to support this hypothesis, we have implemented an approximated version of our model as follows. We fixed the bottom layers of our model and recomputed just the top layer, by masking already visited nodes and adding the updated information (origin and destination tokens for TSP).  
As expected, inference time is 1.6 times shorter, but performance is severely degraded: we obtained optimality gap of 8.175\% (vs 0.35\% with original model) on TSP100.

%%%%%%%%%%%%%%%%%%%%%%%%%%%%%%%%%%%%%%%%%%%%%%%%%%%%%%%%%%%%%%%%%%%%%%%%%%%
\subsection{On the impact of expert trajectories}
\label{anx:expert-solutions}
%%%%%%%%%%%%%%%%%%%%%%%%%%%%%%%%%%%%%%%%%%%%%%%%%%%%%%%%%%%%%%%%%%%%%%%%%%%
Our datasets consist of pairs of a problem instance and a solution. 
For imitation learning, we need pairs of a problem instance and an expert trajectory in the MDP. However multiple trajectories may be obtained from the solution. For example, in the TSP, a solution is a loop in a graph, and one has to decide at which node its construction started and in which direction it proceeded. In the CVRP, the order in which the subtours are constructed needs also to be decided. Hence, all our datasets are pre-processed to transform solutions into corresponding construction trajectories (a choice for each or even all possible ones). Our experiments demonstrate that this transformation has a significant impact on performance. Specifically, in the CVRP, we found that the best performance is achieved by training the model on expert solutions that sort subtours by the remaining vehicle capacity at the end of each subtour. More precisely, the last subtour in the expert trajectory has the biggest remaining capacity (the subtour that visits remaining unvisited nodes), while the first subtour had the smallest remaining capacity (usually 0). This simple data preprocessing step leads to an almost twofold improvement in performance compared to training on expert trajectories with subtours in arbitrary order. Intuitively, these trajectories encourage the model to create subtours that use the whole vehicle capacity whenever possible.
\clearpage
%%%%%%%%%%%%%%%%%%%%%%%%%%%%%%%%%%%%%%%%%%%%%%%%%%%%%%%%%%%%%%%%%%%%%%%%%%%
\section{Proofs}
%%%%%%%%%%%%%%%%%%%%%%%%%%%%%%%%%%%%%%%%%%%%%%%%%%%%%%%%%%%%%%%%%%%%%%%%%%%
\subsection{Properties of the direct MDP}
\label{anx:proof-properties}
%%%%%%%%%%%%%%%%%%%%%%%%%%%%%%%%%%%%%%%%%%%%%%%%%%%%%%%%%%%%%%%%%%%%%%%%%%%
\begin{proof}
Let $\MX$ be a solution space, $(f,X){\in}\MF_\MX$ an instance and $\mathcal{M}_{(f,X)}$ its direct MDP. $\MX$ is normally designed to accommodate all the instances of a CO problem, possibly even of multiple problems, so we do not make the assumption that $\MX$ or $\MZ$ are finite, but only that the conditions of Def.~\ref{def:solution-space} are satisfied.

{\bf The empty partial solution $\epsilon$ is a valid state}\\
By definition, $(f,X){\in}\MF_\MX$ implies $X{\not=}\emptyset$, which is equivalent, by definition, to $\epsilon{\in}\bar{X}$, i.e. $\epsilon$ is in the state space of $\mathcal{M}_{(f,X)}$.

{\bf Each state has a finite, non null number of allowed actions}.\\
Let $x{\in}\bar{X}$ be a state of $\mathcal{M}_{(f,X)}$.

By definition $x{\circ}y{\in}X$ for some $y{\in}\MX$, and by \eqref{eqn:state-decomposition}, we have $y{=}z_1{\circ}\cdots{\circ}z_n$ for some $z_{1:n}{\in}\MZ$. If $n{>}0$, then, by associativity, $(x{\circ}z_1){\circ}z_2{\circ}\cdots{\circ}z_n{=}x{\circ}y{\in}X$, hence by definition $x{\circ}z_1{\in}\bar{X}$, hence step action $z_1$ is allowed from $x$. If $n{=}0$ then $y{=}\epsilon$ and $x{=}x{\circ}y{\in}X$ hence the null action $\epsilon$ is allowed from $x$. In both cases, $x$ is not a dead end state.

Since $X$ is finite and, by \eqref{eqn:state-decomposition}, each of its elements has a finite number of step decompositions, the set $Z$ of steps occurring in at least one step decomposition of an element of $X$ is itself finite, even if $\MZ$ is infinite. Now, assume $z$ is an allowed step action from $x$, hence, by definition $x{\circ}z{\in}\bar{X}$ hence $x{\circ}z{\circ}y{\in}X$ for some $y{\in}\MX$. By \eqref{eqn:state-decomposition}, both $x$ and $y$ have step decompositions, hence $z$ occurs in at least one step decomposition of an element of $X$, i.e. $z{\in}Z$. Hence, all step actions allowed from a valid state are in the finite set $Z$.

{\bf All but a finite number of transitions in a trajectory are null actions}\\
By definition and \eqref{eqn:state-decomposition}, an element of $\bar{X}$ is the composition of a prefix of a step decomposition of an element of $X$. Since $X$ is finite and, by \eqref{eqn:state-decomposition}, each of its elements has a finite number of step decompositions, $\bar{X}$ is finite.

Let $x_0 a_1 x_1 a_2 x_2\cdots$ be an infinite trajectory of $\mathcal{M}_{(f,X)}$. By definition of the allowed transitions, it is easy to show that $x_n{=}x_0{\circ}a_1{\circ}\cdots{\circ}a_n$ for all $n{\in}\mathbb{N}$ and furthermore $x_n{\in}\bar{X}$. Let $N{=}\{n{\in}\mathbb{N}|a_n{\not=}\epsilon\}$. Reason by contradiction and assume $N$ is infinite, i.e. there exists an increasing sequence $n_{1:\infty}$ such that $N{=}\{n_i\}_{i=1}^\infty$. For all $i{\geq}1$, we have $x_{n_i}{=}x_0{\circ}z_1{\circ}\cdots{\circ}z_i$ where $z_i{=}a_{n_i}{\in}\MZ$ and $x_{n_i}{\in}\bar{X}$. Since $\bar{X}$ is finite, there exist indices $j{>}i$ such that $x_{n_j}{=}x_{n_i}$. Hence $x_{n_i}{=}x_{n_i}{\circ}z_{i+1}{\circ}\cdots{\circ}z_j$. Hence, any step decomposition of $x_{n_i}$ can be expanded with the non empty sequence $z_{i+1:j}$ and remain a step decomposition of $x_{n_i}$. Hence there are infinitely many step decompositions of $x_{n_i}$. Contradiction. Therefore, $N$ is finite.
\end{proof}
%%%%%%%%%%%%%%%%%%%%%%%%%%%%%%%%%%%%%%%%%%%%%%%%%%%%%%%%%%%%%%%%%%%%%%%%%%%
\subsection{Soundness of the direct MDP}
\label{anx:proof-soundness}
%%%%%%%%%%%%%%%%%%%%%%%%%%%%%%%%%%%%%%%%%%%%%%%%%%%%%%%%%%%%%%%%%%%%%%%%%%%
\begin{proof}
We first show the following general lemma. Let $Y{\stackrel{\psi}{\rightarrow}}X{\stackrel{f}{\rightarrow}}\mathbb{R}{\cup}\{\infty\}$ be arbitrary mappings. If $\psi$ is surjective then
\begin{equation}
\label{eqn:soundness}
\arg\min_{x\in X}f(x)=\psi(\arg\min_{y\in Y}f(\psi(y)))
\end{equation}
This is shown by simple application of the definition of $\arg\min$ (as a set). The subscript $_*$ denotes the steps where the assumption that $\psi$ is a surjection is used:
\[
\begin{array}{l}
x'\in\psi(\arg{\min}_yf(\psi(y)))
\;\;\;\textrm{iff}\;\;\;
\exists y'\in\arg{\min}_yf(\psi(y))\; x'=\psi(y')
\\ \textrm{iff}\;\;\;
\exists y'\;x'=\psi(y')\;\forall y\;f(\psi(y'))\leq f(\psi(y))
\;\;\;\textrm{iff}\;\;\;
\exists y'\;x'=\psi(y')\;\forall y\;f(x')\leq f(\psi(y))
\\ \textrm{iff$_*$}\;\;\;
\forall y\;f(x') \leq f(\psi(y))
\;\;\;\textrm{iff$_*$}\;\;\;
\forall x\;f(x')\leq f(x)
\;\;\;\textrm{iff}\;\;\;
x'\in\arg{\min}_xf(x)
\end{array}
\]
Now, let $\MX$ be a solution space, $(f,X){\in}\MF_\MX$ an instance and $\mathcal{M}_{(f,X)}$ its direct MDP. Let $Y$ be the set of trajectories in $\mathcal{M}_{(f,X)}$ starting at $\epsilon$, and for each $y{\in}Y$, let $\psi(y)$ denote its outcome.

Observe that $\psi{:}Y{\mapsto}X$. Indeed, it has been shown above that a valid trajectory $y$ always ends with an infinite (stationary) sequence of null transitions on a state $x$ which is also its outcome when $y$ starts with $\epsilon$, i.e. $x{=}\psi(y)$. For the null transitions to be allowed, we must have $x{\in}X$.
\item 
Let's show that $\psi$ is surjective. Let $x{\in}X$. By \eqref{eqn:state-decomposition}, $x{=}z_1{\circ}\cdots{\circ}z_n$ for some $z_{1:n}{\in}\MZ$. For each $m{\in}\{0{:}n\}$, let $x_m{=}z_1{\circ}\cdots{\circ}z_m$ and $x'_m{=}z_{m+1}{\circ}\cdots{\circ}z_n$. Hence $x_m{\circ}x'_m{=}x{\in}X$ hence $x_m{\in}\bar{X}$ is a valid state. Hence, the sequence $y{=}x_0z_1x_1\cdots z_nx_n(\epsilon x_n)^*$ is a valid trajectory of $\mathcal{M}_{(f,X)}$, and it is starting at $x_0{=}\epsilon$, hence $y{\in}Y$. Its stationary state is $x_n{=}x$ hence $\psi(y){=}x$. Hence $\psi$ is surjective.

Now, let $y{=}x_0a_1x_1\cdots a_nx_n\cdots$ be a trajectory in $Y$. By definition, $x_0{=}\epsilon$. Furthermore, by definition of the direct MDP transitions, the reward for $a_n$ (whether it is a step or null) is $f(x_{n-1}){-}f(x_{n})$, which is null when the trajectory becomes stationary on state $\psi(y)$. By summation, the total reward $R(y)$ of trajectory $y$ is $f(\epsilon){-}f(\psi(y))$. The objective is defined up to an additive constant, so we can assume without loss of generality that $f(\epsilon){=}0$. Hence $f(\psi(y)){=}{-}R(y)$.

Finally, since $\psi$ is surjective, we can apply \eqref{eqn:soundness} proved above and get
\[
\arg{\min}_{x\in X}f(x)
= \psi(\arg{\min}_{y\in Y}f(\psi(y)))
=\psi(\arg{\min}_{y\in Y}{-}R(y))
=\psi(\arg{\max}_{y\in Y}R(y))
\]
In other words, an optimal solution to $(f,X)$ is the outcome of an optimal trajectory of $\mathcal{M}_{(f,X)}$ starting at $\epsilon$, i.e. one which can be obtained by application of an optimal policy.
\end{proof}
%%%%%%%%%%%%%%%%%%%%%%%%%%%%%%%%%%%%%%%%%%%%%%%%%%%%%%%%%%%%%%%%%%%%%%%%%%%
\subsection{Bisimulation between the direct MDP and the BQ-MDP}
\label{anx:proof-bisimulation}
%%%%%%%%%%%%%%%%%%%%%%%%%%%%%%%%%%%%%%%%%%%%%%%%%%%%%%%%%%%%%%%%%%%%%%%%%%%
Let $\MX$ be a solution space, $(f,X){\in}\MF_\MX$ be an instance, $x{\in}\bar{X}$ be a valid state for the direct MDP of $(f,X)$. We have to show that $\boldsymbol{\Phi}_{(f,X)}$ is a bisimulation $\mathcal{M}_{(f,X)}{\leftrightarrow}\mathcal{M}$, i.e. the commutation of the diagram (see Sec.~\ref{anx:background-bisimulation} for background):
\begin{center}
\begin{tikzpicture}
\node (p) at (0,0) {$x$};\node (q) at (0,-1.5) {$(f*x,X*x)$};
\node (pp) at (8,0) {$x\circ z$};\node (qq) at (8,-1.5) {$T$};
\draw[->] (p) -- node [midway,above] {$z$} node[midway,below] {$f(x)-f(x\circ z)$} (pp);
\draw[->] (q) -- node [midway,above] {$z$} node [midway,below] {$(f*x)(\epsilon)-(f*x)(z)$} (qq);
\draw[->] (p) -- (q) node [midway,left] {$\boldsymbol{\Phi}_{(f,X)}$};
\draw[->] (pp) -- (qq) node [midway,right] {$\boldsymbol{\Phi}_{(f,X)}$};
\end{tikzpicture}
\end{center}
i.e. the values of $T$ (SE corner) obtained from $x$ (NW corner) via the two paths (NW-SW-SE and NW-NE-SE) are the same.
\begin{proof}
The value of $T$ via NW-SW-SE is given by $((f{*}x){*}z,(X{*}x){*}z)$, while via NW-NE-SE it is $(f{*}(x{\circ}z),X{*}(x{\circ}z))$. For any $y{\in}\MX$ we have, by associativity
\[
(f*(x\circ z))(y) = f((x\circ z)\circ y) = f(x\circ(z\circ y)) = (f*x)(z\circ y) = ((f*x)*z)(y)
\]
Hence $f{*}(x{\circ}z){=}(f{*}x){*}z$. The proof that $X{*}(x{\circ}z){=}(X{*}x){*}z$ is identical. Hence $T$ has the same result via the two paths. 

There are two more things to verify. First that the reward of the N and S transitions are the same: this is obvious since by definition $(f{*}x)(\epsilon){=}f(x{\circ}\epsilon){=}f(x)$ and $(f{*}x)(z){=}f(x{\circ}z)$. Second, that the conditions for action $z$ to be allowed in the N and S transitions are the same: this is also obvious since the condition in the N transition is $x{\circ}z{\in}\bar{X}$ while for the S transition, it is $(X{*}x){*}z{\not=}\emptyset$ or equivalently $X{*}(x{\circ}z){\not=}\emptyset$, and we have for any $y{\in}\MX$ (a fortiori for $y{=}x{\circ}z$)
\[
X*y\not=\emptyset
\textrm{ iff }
\exists y'\;y\circ y'\in X
\textrm{ iff }
y\in\bar{X}
\]

The other commutation, for the null action, is obvious:
\begin{center}
\begin{tikzpicture}
\node (p) at (0,0) {$x$};\node (q) at (0,-1.5) {$(f*x,X*x)$};
\node (pp) at (3,0) {$x$};\node (qq) at (3,-1.5) {$(f*x,X*x)$};
\draw[->] (p) -- node [midway,above] {$\epsilon$} node[midway,below] {$0$} (pp);
\draw[->] (q) -- node [midway,above] {$\epsilon$} node [midway,below] {$0$} (qq);
\draw[->] (p) -- (q) node [midway,left] {$\boldsymbol{\Phi}_{(f,X)}$};
\draw[->] (pp) -- (qq) node [midway,right] {$\boldsymbol{\Phi}_{(f,X)}$};
\end{tikzpicture}
\end{center}
since the condition for $\epsilon$ to be allowed in the N transition is $x{\in}X$ and in the S transition is $\epsilon{\in}X{*}x$, and the two conditions are equivalent by definition.
\end{proof}
Observe that the core of the proof is that $\MF_\MX$ with operation $*$ is a monoid action on the right for monoid $\MX$. Hence, it defines a weaker, discrete generalisation of the mathematical notion of flow (a group action for the group of real numbers). Using Prop.~\ref{prop:bisimulation-disjoint-union}, the above result can be strictly equivalently reformulated as:

Let $\boldsymbol{\Phi}^\sqcup{=}\bigsqcup_{(f,X)\in\MF_\MX}\boldsymbol{\Phi}_{(f,X)}$ and $\mathcal{M}^\sqcup{=}\bigsqcup_{(f,X)\in\MF_\MX}\mathcal{M}_{(f,X)}$ be the disjoint union of the bisimulation mappings $\boldsymbol{\Phi}_{(f,X)}$ (on their domain side) and that of their corresponding direct MDPs $\mathcal{M}_{(f,X)}$, respectively. Then $\boldsymbol{\Phi}^\sqcup$ is a bisimulation $\mathcal{M}^\sqcup{\leftrightarrow}\mathcal{M}$ where $\mathcal{M}$ is the reduced MDP of $\MX$. 

Now, $\boldsymbol{\Phi}^\sqcup$ is obviously surjective, since $\boldsymbol{\Phi}^\sqcup(((f,X),\epsilon)){=}(f,X)$ for any $(f,X){\in}\MF_\MX$. Hence, by Prop.~\ref{prop:bisimulation-quotient}, $\mathcal{M}$ is {\em isomorphic} to the quotient MDP $\mathcal{M}^\sqcup/\boldsymbol{\Phi}^\sqcup$.
%%%%%%%%%%%%%%%%%%%%%%%%%%%%%%%%%%%%%%%%%%%%%%%%%%%%%%%%%%%%%%%%%%%%%%%%%%%
\section{Mathematical background}
%%%%%%%%%%%%%%%%%%%%%%%%%%%%%%%%%%%%%%%%%%%%%%%%%%%%%%%%%%%%%%%%%%%%%%%%%%%
\subsection{Monoids}
\label{anx:background-monoids}
%%%%%%%%%%%%%%%%%%%%%%%%%%%%%%%%%%%%%%%%%%%%%%%%%%%%%%%%%%%%%%%%%%%%%%%%%%%
Monoids are one of the simplest algebraic structure.
\begin{definition}[Monoid]
A monoid is a triple $(M,\circ,\epsilon)$ where $\circ$ is a binary operation on the set $M$ and $\epsilon$ is a distinguished element of $M$, such that
\begin{align*}
\forall x,y,z\in M\hspace{1cm} &
x\circ (y\circ z) = (x\circ y) \circ z
\hspace{1cm}\textrm{(associativity)}\\
\forall x \in M\hspace{1cm} &
x\circ\epsilon = \epsilon\circ x = x
\hspace{1cm}\textrm{(neutral element)}
\end{align*}
\end{definition}
For example the (finite) sequences of elements of an arbitrary set, equipped with concatenation and the empty sequence, forms a monoid. Now, if $M$ is a monoid and $x_{1:n}$ is a sequence of elements of $M$, then the expression $x_1{\circ}\cdots{\circ}x_n$ denotes an element of $M$, independent of the way it is parenthesised. This extends to the case $n{=}0$ where the expression denotes $\epsilon$. And the mapping $x_{1:n}{\mapsto}x_1{\circ}\cdots{\circ}x_n$ from the monoid of sequences of elements of $M$ into $M$ is a monoid homomorphism.

A sub-monoid of $M$ is a subset of $M$ which contains $\epsilon$ and is closed under operation $\circ$. Obviously, a sub-monoid of $M$ is itself a monoid. If $S$ is a sub-monoid of $M$, a generator of $S$, if it exists, is a subset $S_o$ of $M{\setminus}\{\epsilon\}$ such that $S$ is the smallest sub-monoid of $M$ containing $S_o$. It is easy to show that in that case, $S$ is exactly the set of elements of the form $x_1{\circ}\cdots{\circ}x_n$ where $x_{1:n}{\in}S_o^n$.
%%%%%%%%%%%%%%%%%%%%%%%%%%%%%%%%%%%%%%%%%%%%%%%%%%%%%%%%%%%%%%%%%%%%%%%%%%%
\subsection{Bisimulation}
\label{anx:background-bisimulation}
%%%%%%%%%%%%%%%%%%%%%%%%%%%%%%%%%%%%%%%%%%%%%%%%%%%%%%%%%%%%%%%%%%%%%%%%%%%
Bisimulation is a very broad concept which applies to arbitrary Labelled Transition Systems (LTS). It has been declined in various flavours of LTS, such as Process Calculi, Finite State Automata, Game theory, and of course MDPs (initially deterministic MDPs such as those used here, later extended to stochastic MDPs which we are not concerned with here). We use the following notation to indicate that the transition from state $p$ to state $p'$ with label $\ell$ is valid in LTS $\mathcal{L}$ ($\mathcal{L}$ is omitted when unambiguous).
\[p\xrightarrow[(\mathcal{L})]{\;\;\;\ell\;\;\;}p'\]

Recall that the disjoint union of a family $(S_i)_{i\in I}$ of sets is the set $\bigsqcup_{i\in I}S_i\bydef{=}\bigcup_{i\in I}\{i\}{\times}S_i$. We can define the disjoint union of a family of LTSs as follows:
\begin{definition}[Disjoint union]
If $(\mathcal{L}_i)_{i\in I}$ is a family of LTSs {\em sharing the same label space}, each with state space $S_i$, then the {\em disjoint union} $\bigsqcup_{i\in I}\mathcal{L}_i$ is the LTS $\mathcal{L}$ with state space $\bigsqcup_{i\in I}S_i$ and transitions
\[
(i,p)\xrightarrow[(\mathcal{L})]{\ell}(i,p')
\hspace{1cm}\textrm{if}\hspace{1cm}
p\xrightarrow[(\mathcal{L}_i)]{\ell}p'
\]
\end{definition}
\begin{definition}[Simulation, Bisimulation]
Let $\mathcal{L}_1,\mathcal{L}_2$ be LTSs {\em sharing the same label space} and $\mathcal{R}$ a binary relation relating the states of $\mathcal{L}_1$ and those of $\mathcal{L}_2$. $\mathcal{R}$ is a {\em simulation} $\mathcal{L}_1{\rightarrow}\mathcal{L}_2$ if
\[
\forall \ell,p,q,p' \textrm{ s.t. } p\mathcal{R}q,\;p\xrightarrow[(\mathcal{L}_1)]{\ell}p'\hspace{.5cm}
\exists q' \textrm{ s.t. } p'\mathcal{R}q',\;q\xrightarrow[(\mathcal{L}_2)]{\ell}q'
\]
$\mathcal{R}$ is a {\em bisimulation} $\mathcal{L}_1{\leftrightarrow}\mathcal{L}_2$ if $\mathcal{R}$ is a simulation $\mathcal{L}_1{\rightarrow}\mathcal{L}_2$ and $\mathcal{R}^{\textrm{op}}$ is a simulation $\mathcal{L}_2{\rightarrow}\mathcal{L}_1$.
\end{definition}
Informally, a simulation (resp. bisimulation) is characterised by a commutation property in the following diagram: if the pair of arrows connected to $p$ (resp. to either $p$ or $q$) is valid then so is the ``opposite'' pair w.r.t. the centre of the diagram.
\begin{center}
\begin{tikzpicture}
\node (p) at (0,0) {$p$};\node (q) at (0,-1) {$q$};
\node (pp) at (2,0) {$p'$};\node (qq) at (2,-1) {$q'$};
\draw[->] (p) -- (pp) node [midway,above] {$\ell$};
\draw[->] (q) -- (qq) node [midway,above] {$\ell$};
\draw[->] (p) -- (q) node [midway,left] {$\mathcal{R}$};
\draw[->] (pp) -- (qq) node [midway,right] {$\mathcal{R}$};
\end{tikzpicture}
\end{center}
A {\em homogeneous} simulation (resp. bisimulation) on an LTS $\mathcal{L}$ is a simulation $\mathcal{L}{\rightarrow}\mathcal{L}$ (resp. bisimulation $\mathcal{L}{\leftrightarrow}\mathcal{L}$). Note that $\mathcal{R}$ is a simulation $\mathcal{L}_1{\rightarrow}\mathcal{L}_2$ (resp. bisimulation $\mathcal{L}_1{\leftrightarrow}\mathcal{L}_2$) if and only if $\{((1,p),(2,q))|(p,q){\in}\mathcal{R}\}$ is a homogeneous simulation (resp. bisimulation) on $\mathcal{L}_1{\sqcup}\mathcal{L}_2$.
\begin{proposition}
\label{prop:bisimulation-disjoint-union}
Let $\mathcal{L},(\mathcal{L}_i)_{i\in I}$ be LTSs {\em sharing the same label space}. For each $i{\in}I$, let $\mathcal{R}_i$ be a bi-partite relation from the state space of $\mathcal{L}_i$ into that of $\mathcal{L}$, and let $\mathcal{R}^\sqcup{=}\{((i,p),q)|(p,q){\in}\mathcal{R}_i\}$ be their disjoint union on the domain side. Then $\mathcal{R}^\sqcup$ is a bisimulation $(\bigsqcup_{i\in I}\mathcal{L}_i){\leftrightarrow}\mathcal{L}$ if and only if $\mathcal{R}_i$ is a bisimulation $\mathcal{L}_i{\leftrightarrow}\mathcal{L}$ for each $i{\in}I$.
\end{proposition}
\begin{proof}
(outline) This is essentially shown by observing that the following two diagrams, where $\mathcal{L}^\sqcup{=}\bigsqcup_{i\in I}\mathcal{L}_i$, are, by definition, equivalent:
\begin{center}
\begin{tikzpicture}
\node (p) at (0,0) {$(i,p)$};\node (q) at (0,-1) {$q$};
\node (pp) at (2,0) {$(i,p')$};\node (qq) at (2,-1) {$q'$};
\draw[->] (p) -- (pp) node [midway,below] {\scriptsize $(\mathcal{L}^\sqcup)$} node [midway,above] {$\ell$};
\draw[->] (q) -- (qq) node [midway,below] {\scriptsize $(\mathcal{L})$} node [midway,above] {$\ell$};
\draw[->] (p) -- (q) node [midway,left] {$\mathcal{R}^\sqcup$};
\draw[->] (pp) -- (qq) node [midway,right] {$\mathcal{R}^\sqcup$};
\begin{scope}[xshift=6cm]
\node (p) at (0,0) {$p$};\node (q) at (0,-1) {$q$};
\node (pp) at (2,0) {$p'$};\node (qq) at (2,-1) {$q'$};
\draw[->] (p) -- (pp) node [midway,below] {\scriptsize $(\mathcal{L}_i)$} node [midway,above] {$\ell$};
\draw[->] (q) -- (qq) node [midway,below] {\scriptsize $(\mathcal{L})$} node [midway,above] {$\ell$};
\draw[->] (p) -- (q) node [midway,left] {$\mathcal{R}_i$};
\draw[->] (pp) -- (qq) node [midway,right] {$\mathcal{R}_i$};
\end{scope}
\end{tikzpicture}
\end{center}
If the commutation property holds for one it holds for the other.
\end{proof}
\begin{proposition}
\label{prop:bisimulation-stability}
The identity on the state space of $\mathcal{L}$ is a bisimulation $\mathcal{L}{\leftrightarrow}\mathcal{L}$. The composition of a bisimulation $\mathcal{L}_1{\leftrightarrow}\mathcal{L}_2$ and a bisimulation $\mathcal{L}_2{\leftrightarrow}\mathcal{L}_3$ is a bisimulation $\mathcal{L}_1{\leftrightarrow}\mathcal{L}_3$. The union of a family of bisimulations $\mathcal{L}_1{\leftrightarrow}\mathcal{L}_2$ is a bisimulation $\mathcal{L}_1{\leftrightarrow}\mathcal{L}_2$. The inverse of a bisimulation $\mathcal{L}_1{\leftrightarrow}\mathcal{L}_2$ is a bisimulation $\mathcal{L}_2{\leftrightarrow}\mathcal{L}_1$.
\end{proposition}
Hence, LTSs with bisimulation form a category in Category theory.
\begin{proof}
(outline) Let's detail stability by composition, the other cases are similarly obvious. If $\mathcal{R}_1,\mathcal{R}_2$ are the two bisimulations being composed, apply the commutation property to each cell of the following diagram (from top to bottom and vice versa).
\begin{center}
\begin{tikzpicture}
\node (p) at (0,0) {$p$};\node (r) at (0,-1) {$r$};\node (q) at (0,-2) {$q$};
\node (pp) at (2,0) {$p'$};\node (rr) at (2,-1) {$r'$};\node (qq) at (2,-2) {$q'$};
\draw[->] (p) -- (pp) node [midway,above] {$\ell$};
\draw[->] (q) -- (qq) node [midway,above] {$\ell$};
\draw[->] (r) -- (rr) node [midway,above] {$\ell$};
\draw[->] (p) -- (r) node [midway,left] {$\mathcal{R}_1$};
\draw[->] (r) -- (q) node [midway,left] {$\mathcal{R}_2$};
\draw[->] (pp) -- (rr) node [midway,right] {$\mathcal{R}_1$};
\draw[->] (rr) -- (qq) node [midway,right] {$\mathcal{R}_2$};
\end{tikzpicture}
\end{center}
\end{proof}
As a corollary, observe that the set of homogeneous bisimulations on an LTS $\mathcal{L}$ is stable by reflexive-symmetric-transitive closure. In particular, the union of all bisimulations, called the {\em bisimilarity} of $\mathcal{L}$, is itself a bisimulation, and it is an equivalence relation.
\begin{definition}[Quotienting]
Given an LTS $\mathcal{L}$ and an equivalence relation $\mathcal{R}$ on its state space, the quotient LTS $\mathcal{L}/\mathcal{R}$ is defined as follows: the label space is the same as that of $\mathcal{L}$; the states are the $\mathcal{R}$-equivalence classes; and the transitions are defined, for any classes $\dot{p},\dot{p}'$, by
\[
\dot{p}\xrightarrow[\;\;(\mathcal{L}/\mathcal{R})\;\;]{\ell}\dot{p}'
\hspace{.5cm}\textrm{if}\hspace{.5cm}
\forall p\in\dot{p}\;\exists p'\in\dot{p}'\;\;\;p\xrightarrow[\;\;(\mathcal{L})\;\;]{\ell}p'
\]
By extension, if $F$ is a mapping from the state space of $\mathcal{L}$ to an arbitrary set, then $\mathcal{L}/F$ denotes $\mathcal{L}/(F^{\textrm{op}}{\circ}F)$.
\end{definition}
\begin{proposition}
\label{prop:bisimulation-quotient-base}
Let $\mathcal{R}$ be an equivalence on the state space of $\mathcal{L}$. $\mathcal{R}$ is a (homogeneous) bisimulation on $\mathcal{L}$ if and only if $\in$ is a bisimulation $\mathcal{L}{\leftrightarrow}\mathcal{L}/\mathcal{R}$.
\end{proposition}
\begin{proof}
We show both implications: first, assume $\mathcal{R}$ is a bisimulation on $\mathcal{L}$.

{\bf 1--} Let $p\in\dot{q}$ and $p\xrightarrow{\ell}p'$. Let $q\in\dot{q}$. Hence $p\mathcal{R}q$ and $p\xrightarrow{\ell}p'$. Since $\mathcal{R}$ is a bisimulation, there exists $q'$ such that $q\xrightarrow{\ell}q'$ and $p'\mathcal{R}q'$. Hence for all $q\in\dot{q}$, there exists $q'\in\bar{p}'$ such that $q\xrightarrow{\ell}q'$. Hence by definition $\dot{q}\xrightarrow{\ell}\bar{p}'$ while $p'\in\bar{p}'$.

{\bf 2--} Let $p\in\dot{q}$ and $\dot{q}\xrightarrow{\ell}\dot{q}'$. Hence by definition, there exists $p'\in\dot{q}'$ such that $p\xrightarrow{\ell}p'$.

Conversely, assume $\in$ is a bisimulation $\mathcal{L}{\leftrightarrow}\mathcal{L}/\mathcal{R}$.

{\bf 1--} Let $p\mathcal{R}q$ and $p\xrightarrow{\ell}p'$. Hence $p\in\bar{q}$ and $p\xrightarrow{\ell}p'$. Since $\in$ is a bisimulation, there exists $\dot{q}'$ such that $p'\in\dot{q}'$ and $\bar{q}\xrightarrow{\ell}\dot{q}'$. Now $q\in\bar{q}$, hence, by definition, there exists $q'\in\dot{q}'$ such that $q\xrightarrow{\ell}q'$. And $p'\mathcal{R}q'$ since $p',q'\in\dot{q}'$.

{\bf 2--} Let $p\mathcal{R}q$ and $q\xrightarrow{\ell}q'$. Hence $q\mathcal{R}p$ and $q\xrightarrow{\ell}q'$, and we are in the previous case up to a permutation of variables.
\end{proof}
\begin{proposition}[Bisimulation Quotienting]
\label{prop:bisimulation-quotient}
Let $F$ be a mapping from the state space of $\mathcal{L}_1$ to that of $\mathcal{L}_2$. If $F$ is a surjective bisimulation $\mathcal{L}_1{\leftrightarrow}\mathcal{L}_2$, then there exists a {\em bijective} (a.k.a. {\em isomorphic}) bisimulation $(\mathcal{L}_1/F){\leftrightarrow}\mathcal{L}_2$.
\end{proposition}
\begin{proof}
$F$ is a bisimulation $\mathcal{L}_1{\leftrightarrow}\mathcal{L}_2$, hence, by Prop.~\ref{prop:bisimulation-stability}, $F^{\textrm{op}}$ is a bisimulation $\mathcal{L}_2{\leftrightarrow}\mathcal{L}_1$ and $\mathcal{R}{=}F^{\textrm{op}}{\circ}F$ is a bisimulation $\mathcal{L}_1{\leftrightarrow}\mathcal{L}_1$. Hence, by Prop.~\ref{prop:bisimulation-quotient-base}, $\in$ is a bisimulation $\mathcal{L}_1{\leftrightarrow}(\mathcal{L}_1/F)$ and, by Prop.~\ref{prop:bisimulation-stability} again, $\tilde{F}{=}F{\circ}{\in}^{\textrm{op}}$ is a bisimulation $(\mathcal{L}_1/F){\leftrightarrow}\mathcal{L}_2$. By construction, $\tilde{F}$ is an injective mapping between the state space of $\mathcal{L}_1/F$ and that of $\mathcal{L}_2$. Indeed

{\bf 1--} Consider a class $\dot{p}$ modulo $\mathcal{R}$ and let $p{\in}\dot{p}$. Hence, by definition, $\dot{p}\tilde{F}q$ where $q{=}F(p)$. Furthermore, suppose $\dot{p}\tilde{F}q_1$ and $\dot{p}\tilde{F}q_2$. Hence $q_1{=}F(p_1)$ and $q_2{=}F(p_2)$ for some $p_1,p_2{\in}\dot{p}$. Hence, by definition, $p_1\mathcal{R}p_2$ hence $F(p_1){=}F(p_2)$, i.e. $q_1{=}q_2$. Hence $\tilde{F}$ is a mapping.

{\bf 2--} Suppose $\dot{p}_1\tilde{F}q$ and $\dot{p}_2\tilde{F}q$. Hence $F(p_1){=}F(p_2){=}q$ for some $p_1,p_2{\in}\dot{p}_1{\times}\dot{p}_2$. Hence, by definition, $p_1\mathcal{R}p_2$, i.e. $\dot{p}_1{=}\dot{p}_2$. Hence $\tilde{F}$ is injective.

Now, ${\in}^{\textrm{op}}$ is obviously surjective between the state space of $\mathcal{L}_1/F$ and that of $\mathcal{L}_1$, hence when $F$ is surjective, so is $\tilde{F}$, which is then bijective.
\end{proof}
\begin{definition}
\label{def:lts-closure}
Given an LTS $\mathcal{L}$, its {\em transitive closure} is another LTS denoted $\mathcal{L}^*$ on the same state space, where the labels are the finite sequences of labels of $\mathcal{L}$ and the transitions are defined by
\[
p\xrightarrow[\;(\mathcal{L}^*)\;]{\ell_{1:n}}p'
\hspace{.5cm}\textrm{if}\hspace{.5cm}
\exists p_{0:n}\textrm{ such that }p=p_0\xrightarrow[\;(\mathcal{L})\;]{\;\ell_1\;}p_1\cdots \xrightarrow[\;(\mathcal{L})\;]{\;\ell_{n-1}\;}p_{n-1}\xrightarrow[\;(\mathcal{L})\;]{\;\ell_n\;}p_n=p'
\]
\end{definition}
\begin{proposition}
\label{prop:bisimulation-transitive}
If $\mathcal{R}$ is a bisimulation $\mathcal{L}_1{\leftrightarrow}\mathcal{L}_2$, then it is also a bisimulation $\mathcal{L}^*_1{\leftrightarrow}\mathcal{L}^*_2$.
\end{proposition}
\begin{proof}
(outline) This is essentially shown by successively applying the commutation property to each cell of the following diagram (from left to right):
\begin{center}
\begin{tikzpicture}
\node (p0) at (0,0) {$p_0$};\node (q0) at (0,-1) {$q_0$};
\node (p1) at (2,0) {$p_1$};\node (q1) at (2,-1) {$q_1$};
\node (p2) at (4,0) {$p_{n-1}$};\node (q2) at (4,-1) {$q_{n-1}$};
\node (p3) at (6,0) {$p_n$};\node (q3) at (6,-1) {$q_n$};
\draw[->] (p0) -- (p1) node [midway,above] {$\ell_1$};
\draw[->,dashed] (p1) -- (p2);
\draw[->] (p2) -- (p3) node [midway,above] {$\ell_n$};
\draw[->] (q0) -- (q1) node [midway,above] {$\ell_1$};
\draw[->,dashed] (q1) -- (q2);
\draw[->] (q2) -- (q3) node [midway,above] {$\ell_n$};
\draw[->] (p0) -- (q0) node [midway,left] {$\mathcal{R}$};
\draw[->] (p1) -- (q1) node [midway,left] {$\mathcal{R}$};
\draw[->] (p2) -- (q2) node [midway,right] {$\mathcal{R}$};
\draw[->] (p3) -- (q3) node [midway,right] {$\mathcal{R}$};
\end{tikzpicture}
\end{center}
\end{proof}
\begin{proposition}
\label{prop:bisimulation-transitive-quotient}
Let $\mathcal{R}$ be an equivalence relation on the state space of $\mathcal{L}$. If $\mathcal{R}$ is a bisimulation on $\mathcal{L}$, then $(\mathcal{L}/\mathcal{R})^*{=}\mathcal{L}^*/\mathcal{R}$.
\end{proposition}
\begin{proof}
$\mathcal{R}$ is a bisimulation on $\mathcal{L}$, hence $\in$ is a bisimulation $\mathcal{L}{\leftrightarrow}\mathcal{L}/\mathcal{R}$ (Prop.~\ref{prop:bisimulation-quotient}), hence also a bisimulation $\mathcal{L}^*{\leftrightarrow}(\mathcal{L}/\mathcal{R})^*$ (Prop.~\ref{prop:bisimulation-transitive}). Analogously, $\mathcal{R}$ is a bisimulation on $\mathcal{L}$, hence also a bisimulation on $\mathcal{L}^*$  (Prop.~\ref{prop:bisimulation-transitive}), hence $\in$ is a bisimulation $\mathcal{L}^*{\leftrightarrow}\mathcal{L}^*/\mathcal{R}$ (Prop.~\ref{prop:bisimulation-quotient-base}). By composition (Prop.~\ref{prop:bisimulation-stability}) we have ${\in}{\circ}{\in}^{\textrm{op}}$ is a bisimulation $\mathcal{L}^*/\mathcal{R}{\leftrightarrow}(\mathcal{L}/\mathcal{R})^*$. But ${\in}{\circ}{\in}^{\textrm{op}}$ is the identity. Hence $\mathcal{L}^*/\mathcal{R}{=}(\mathcal{L}/\mathcal{R})^*$.
\end{proof}

\clearpage
%%%%%%%%%%%%%%%%%%%%%%%%%%%%%%%%%%%%%%%%%%%%%%%%%%%%%%%%%%%%%%%%%%%%%%%%%%%
\section{Plots of some TSPLib and CVRPLib solutions}
%%%%%%%%%%%%%%%%%%%%%%%%%%%%%%%%%%%%%%%%%%%%%%%%%%%%%%%%%%%%%%%%%%%%%%%%%%%
\begin{center}
    \includegraphics[scale=0.4]{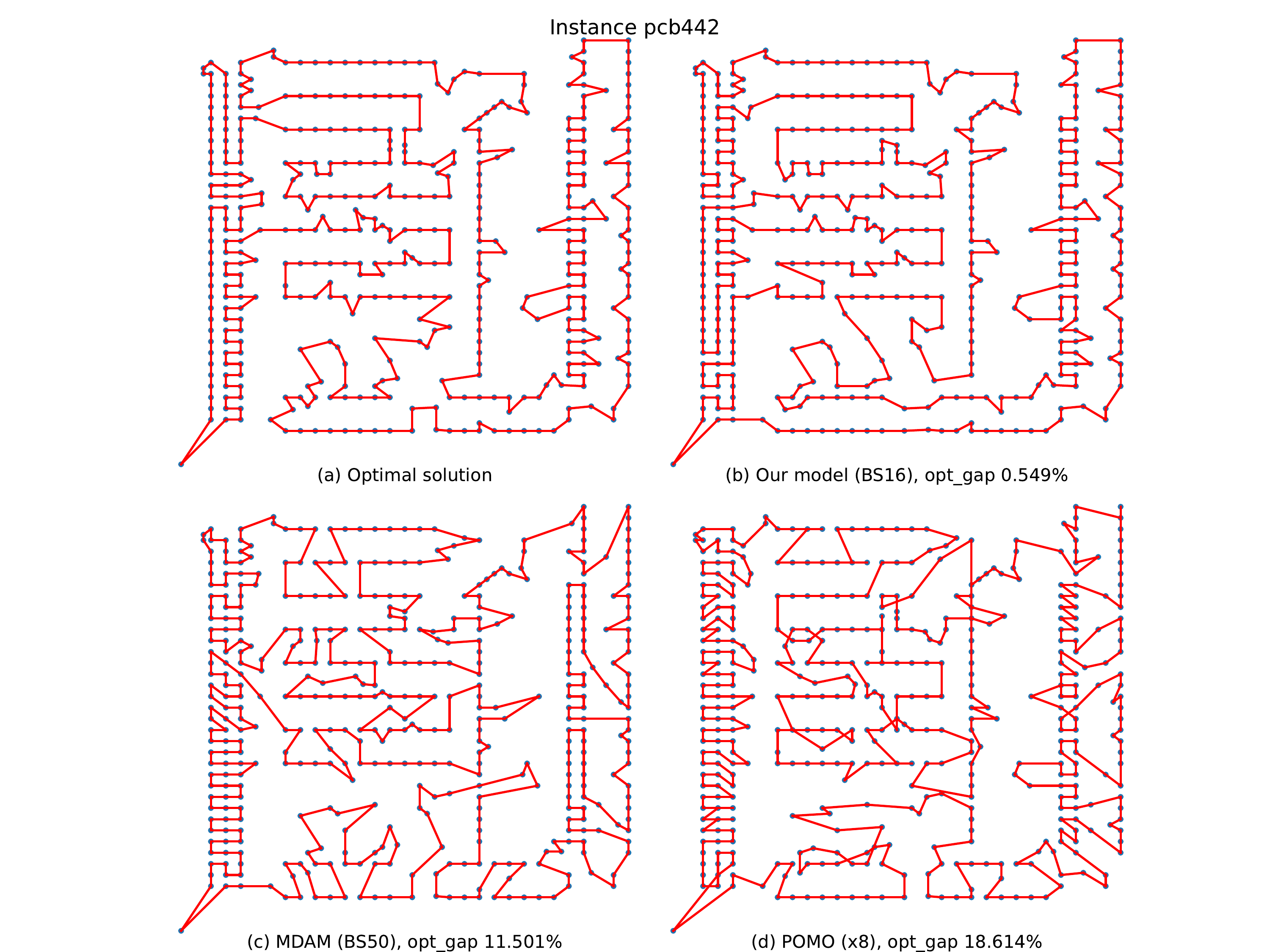}
\end{center}

\begin{center}
    \includegraphics[scale=0.4]{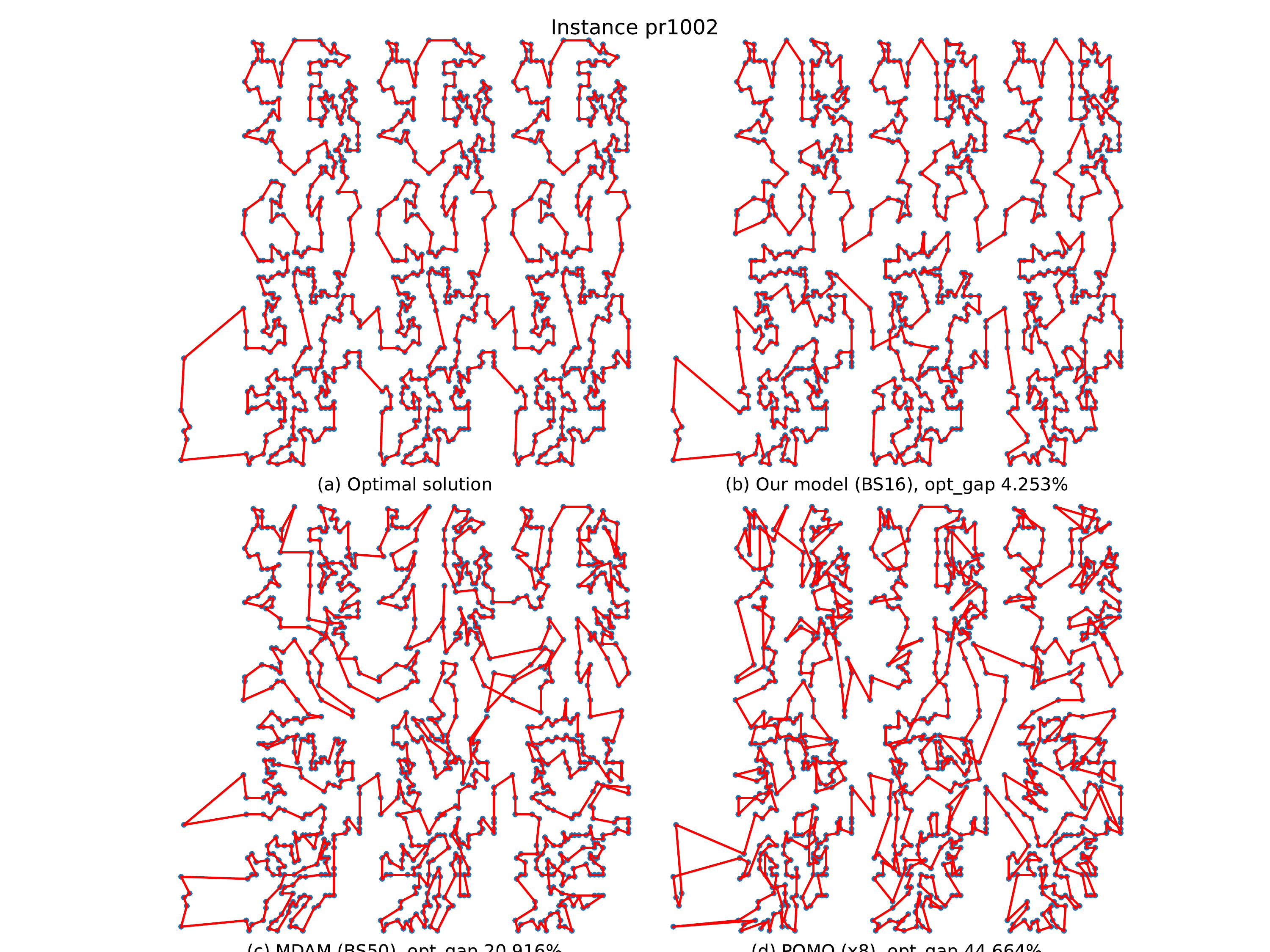}
\end{center}

\begin{center}
    \includegraphics[scale=0.4]{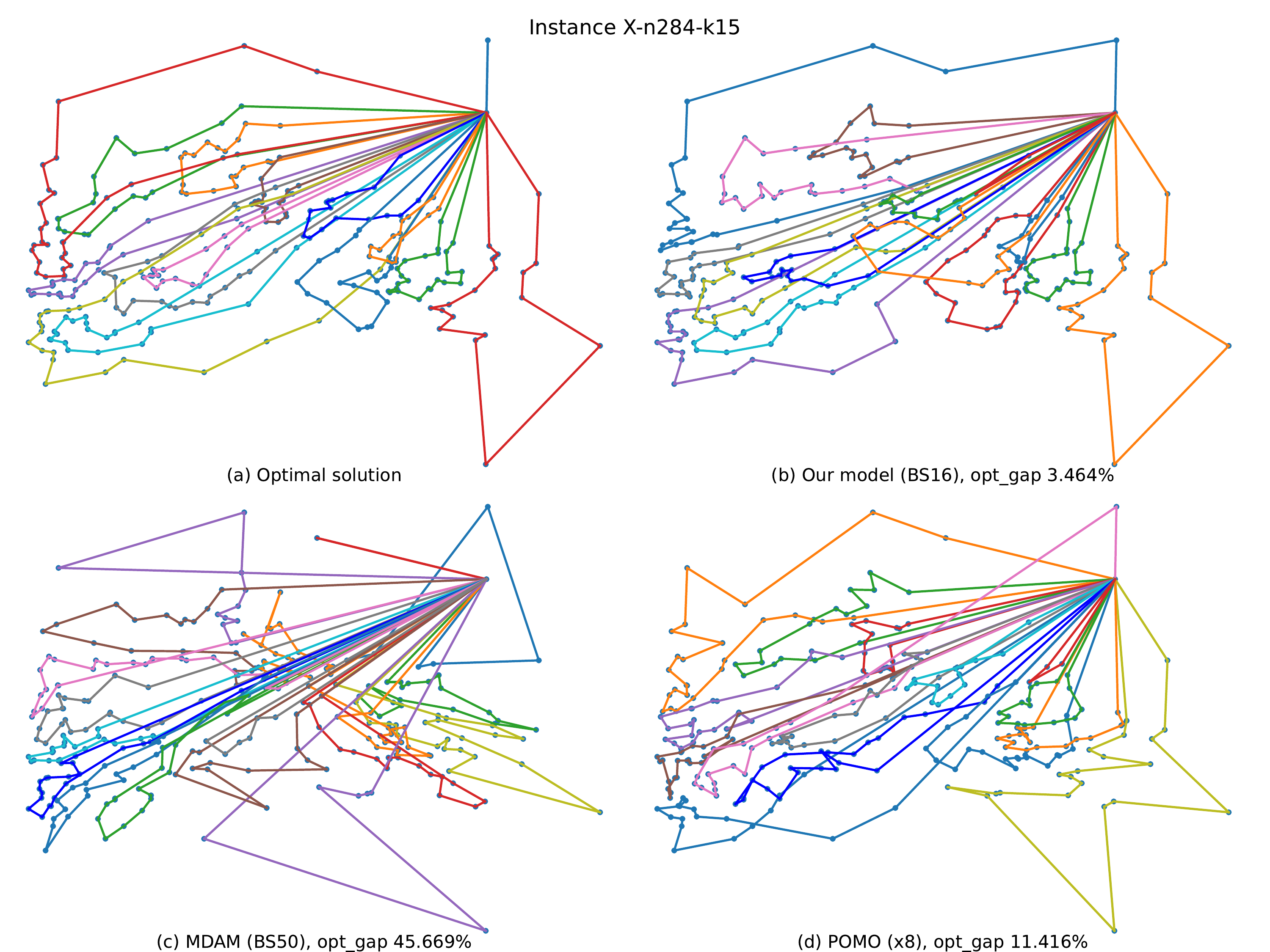}
\end{center}

\begin{center}
    \includegraphics[scale=0.4]{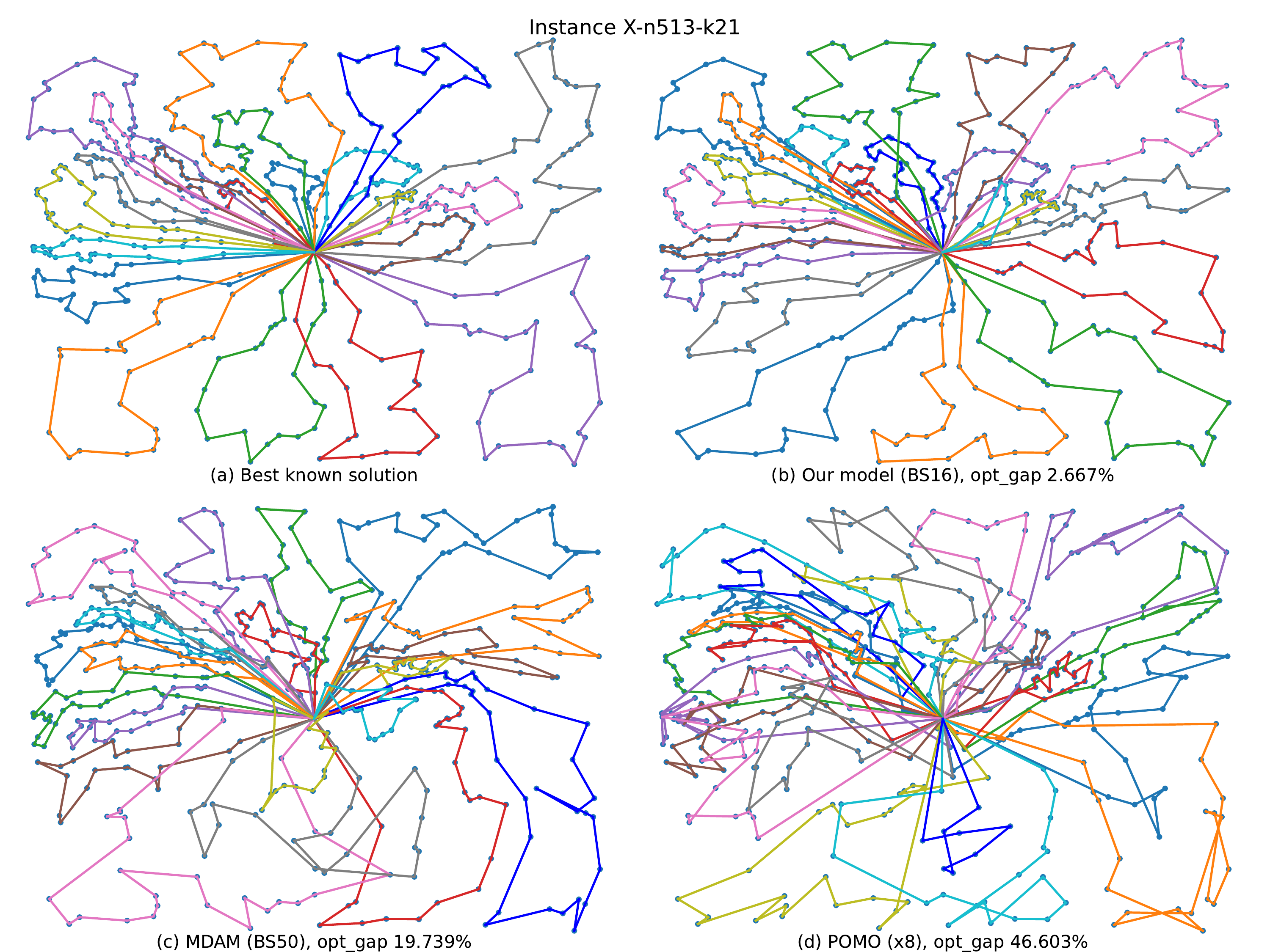}
\end{center}

\end{document}